\documentclass[11pt]{article}
\usepackage{fullpage,url,natbib}
\usepackage{times}

\usepackage{amsthm,amsfonts,amsmath,amssymb,epsfig,color,float,graphicx,verbatim}
\usepackage{enumitem}
\usepackage{varwidth}

\usepackage{algorithm}
\usepackage{algpseudocode}
\usepackage{footmisc}
\usepackage{wrapfig}
\usepackage{subcaption}
\usepackage{bbm, bm}
\usepackage{natbib}
\usepackage[dvipsnames]{xcolor}
\usepackage{titlesec}
\usepackage{setspace}

\usepackage{hyperref}
\hypersetup{
	colorlinks   = true, 
	urlcolor     = blue, 
	linkcolor    = blue!70!black, 
	citecolor   = blue 
}

\newtheorem{theorem}{Theorem}[section]
\newtheorem*{theorem*}{Theorem}
\newtheorem{proposition}[theorem]{Proposition}
\newtheorem*{proposition*}{Proposition}

\newtheorem{lemma}[theorem]{Lemma}
\newtheorem{remark}[theorem]{Remark}

\newtheorem{definition}[theorem]{Definition}
\newtheorem{assumption}[theorem]{Assumption}

\newtheorem{fact}[theorem]{Fact}

\renewcommand{\eqref}[1]{Eq.~(\ref{#1})}

\newcommand{\LB}{\bP_{\mathrm{B}}}  
\newcommand{\LsB}{\bP_{\mathrm{sB}}}  
\newcommand{\LG}{\bP_{\mathrm{G}}}  

\newcommand{\bP}{\mathbf{P}}
\newcommand{\bPmult}{{\bP^{\mathsf{mult}}}}

\newcommand{\mem}{\mathrm{mem}}

\newcommand{\test}{\mathrm{test}}
\newcommand{\train}{\mathrm{train}}
\newcommand{\BSC}{\mathrm{BSC}}

\newcommand{\err}{\mathrm{err}}
\newcommand{\errn}{\mathrm{errn}}

\newcommand{\one}[1]{\mathbbm{1}\left\{#1\right\}}

\newcommand{\reals}{\mathbb{R}}

\newcommand{\R}{\mathbb{R}}

\newcommand{\indep}{\perp\!\!\!\perp}

\newcommand{\TV}{\mathrm{TV}}

\newcommand{\vertiii}[1]{{\left\vert\kern-0.25ex\left\vert\kern-0.25ex\left\vert #1 
    \right\vert\kern-0.25ex\right\vert\kern-0.25ex\right\vert}}

\newcommand{\X}{\mathcal{X}}

\newcommand{\eps}{\varepsilon}

\newcommand{\kl}[2]{D_{
\mathrm{KL}
  }
  \left(#1 \,||\, #2\right)
  }

\usepackage{xcolor}

\newcommand{\beq}{\begin{eqnarray*}}
\newcommand{\eeq}{\end{eqnarray*}}
\newcommand{\beqn}{\begin{eqnarray}}
\newcommand{\eeqn}{\end{eqnarray}}

\newcommand{\Acal}{\mathcal{A}}
\newcommand{\Bcal}{\mathcal{B}}

\newcommand{\Ncal}{\mathcal{N}}

\newcommand{\Ucal}{\mathcal{U}}
\newcommand{\Xcal}{\mathcal{X}}

\newcommand{\NN}{\mathbb{N}}

\newcommand{\half}{\frac{1}{2}}

\newcommand{\E}{\mathop{\mathbb{E}}}

\newcommand{\Mcal}{\mathcal{M}}

\newcommand{\Pcal}{\mathcal{P}}

\newcommand{\inner}[1]{\langle#1\rangle}

\newcommand{\vremove}[1]{}

\title{Trade-offs in Data Memorization
\\via Strong Data Processing Inequalities}
\author{
Vitaly Feldman
\vspace{2pt}
\\Apple
\and
Guy Kornowski
\vspace{2pt}
\\Weizmann Institute of Science\thanks{Research done while at Apple.
}
\and
Xin Lyu
\vspace{2pt}
\\UC Berkeley\footnotemark[1]
}
\date{}

\begin{document}

\maketitle

\begin{abstract}
Recent research demonstrated that training large language models involves memorization of a significant fraction of training data. Such memorization can lead to privacy violations when training on sensitive user data and thus motivates the study of data memorization's role in learning.
In this work, we develop a general approach for proving lower bounds on excess data memorization, that relies on a new connection between strong data processing inequalities and data memorization. We then demonstrate that several simple and natural binary classification problems exhibit a trade-off between the number of samples available to a learning algorithm, and the amount of information about the training data that a learning algorithm needs to memorize to be accurate. In particular, $\Omega(d)$ bits of information about the training data need to be memorized when $O(1)$ $d$-dimensional examples are available, which then decays as the number of examples grows at a problem-specific rate. Further, our lower bounds are generally matched (up to logarithmic factors) by simple learning algorithms. We also extend our lower bounds to more general mixture-of-clusters models. Our definitions and results build on the work of \citet{brown2021memorization} and address several limitations of the lower bounds in their work.
\end{abstract}

\newpage
\tableofcontents

\newpage

\section{Introduction}
The machine learning (ML) methodology is
traditionally thought of as constructing a model by extracting patterns in the training data. Theoretical understanding of machine learning focuses on understanding how to ensure that the constructed model generalizes well from the training data to the unseen instances. In this context, memorization of training data is typically thought of as antithetical to generalization. 
Yet, it has been empirically demonstrated that a variety of modern LLMs memorize a significant portion of training data \citep{carlini2021extracting,carlini2023extracting,nasr2025scalable}. Specifically, \cite{nasr2025scalable} demonstrate an attack allowing them to estimate that at least $0.852\%$ of the training data used by ChatGPT ({\texttt{gpt-3.5-turbo}} used in production by OpenAI as of 2023) can be extracted from the model. 
Importantly, this includes information such as personal addresses and URLs that appears to be both highly sensitive and not particularly relevant to the task of modeling language. In many applications training data includes either personally sensitive information or copyrighted works. This makes such data memorization highly concerning and motivates the research into the role of data memorization in learning. 

Explicit data memorization is known to be a crucial part of some learning algorithm, most notably those based on the nearest neighbor classifier. Further, a significant number of classical and modern works establish theoretical generalization guarantees for such methods (e.g. \citealp{cover1967nearest,biau2015lectures}). It is less clear how data memorization emerges when training NNs and how the resulting models encode training data \citep{radhakrishnan2020overparameterized,zhangidentity2020}. However, in this work we focus not on the mechanics of memorization by specific algorithms, but on the question of whether data memorization is necessary for solving natural learning problems, as opposed to just being an artifact of the choice of the learning algorithm.

This question was first addressed by \cite{brown2021memorization}, who proposed to measure ``irrelevant'' training data memorization as the mutual information between the model and the dataset $I(\Acal(X_{1:n});X_{1:n})$,\footnote{A crucial step in the study of data memorization is capturing the intuitive notion using a formal definition. We discuss some alternative approaches in Section~\ref{sec:related}. We also note that the measure $I(\Acal(X_{1:n});X_{1:n})$ was previously used in the context of generalization and we describe work done in this context in Section~\ref{sec:related}.} where $\Acal$ is the learning algorithm and $X_{1:n}$ is the dataset which consists of $n$ i.i.d. samples from a data distribution. For this notion, they demonstrated existence of a simple
multi-class classification problem over $\{0,1\}^d$, where each
accurate learner needs to memorize a constant fraction of all training data, namely satisfies $I(\Acal(X_{1:n});X_{1:n}) = \Omega(nd)$.

The results of \cite{brown2021memorization} rely on two components,
which we briefly recall to motivate our work.
The first
is the focus on the accuracy for classes that only have a \emph{single} example present in the dataset (referred to as singletons).
This is motivated by the ``long-tail'' view of the data distribution proposed in \citep{feldman2020does},
where it was shown that for data distributions that are long-tailed mixtures of clusters, the accuracy of the learning algorithm
on the tail of the data distribution
is determined by the accuracy on singletons.
The second component is a memorization lower bound for a cluster identification problem,
when given
only a single example of that class.
Each cluster in \citep{brown2021memorization} is distributed uniformly over examples satisfying a $\widetilde O(\sqrt{d})$-sparse boolean conjunction, and the clusters are sufficiently different so that each
cluster
can be accurately classified using a single example. However, as shown therein,
doing so requires $\Omega(d)$ bits of information about the single example, most of which is
``irrelevant'' to the learning problem.
Overall,
this suggests that data memorization is necessary for learning in high-dimensional Boolean settings
with just a single relevant sample.

In this work we aim to develop a more general understanding of data memorization in learning. In particular, we 
address two specific limitations of the lower bounds in \citep{brown2021memorization}. The key limitation is that lower bounds in \citep{brown2021memorization} are tailored to a specific sparse Boolean clustering problem in which clusters are defined by uniform distributions over Boolean conjunctions. Such data distribution are not directly related to the real-valued data representations typically manipulated by neural networks. Thus we aim to develop techniques that apply to large classes of data distribution that include natural data distributions over $\R^d$.

The second limitation is the fragility of the lower bound: it applies only to a single example (per cluster). 
Empirically, it is observed that gathering more data can mitigate memorization, eventually allowing models to forget specific samples \citep{jagielski2023measuring}. Thus going beyond a single example and understanding the trade-off between the number of available examples and memorization for a given data distribution is an important question.

\subsection{Our Contribution}
In this work, we develop a general technique for proving lower bounds on data memorization. Our technique focuses on simple binary classification problems in which the goal is to distinguish points coming from a ``cluster'' from those coming from some fixed ``null'' distribution given a small number of examples. We argue that a complex learning problem over a natural data distribution implicitly involves solving many such classification problems (see Section \ref{sec:motivation} for more details).

Our technique relies on establishing a tight connection between our binary classification setting and {\em strong data processing} inequalities (SDPIs), an important tool in information theory dating back to \citep{ahlswede1976spreading}.
As such SDPIs are known for a relatively limited number of pairs of jointly distributed random variables (referred to as {\em channels} in the context of SDPIs),
we develop an approach based on approximate reductions
that enables applying them in the context of learning from datasets.

We then use our general framework to analyze three natural problem instances: Gaussian cluster identification, Boolean cluster identification and sparse Boolean hypercube identification. The first problem is particularly simple and fundamental: classify an example as either sampled from an isotropic Gaussian around the origin or an isotropic Gaussian around a different point (sufficiently far from the origin). For the first two problems we prove an excess memorization lower bound of $\widetilde \Omega(d/n)$ for any learning algorithm (and, in the Boolean case, $n \lesssim  \sqrt{d}$). We further show that this lower bound is tight up to logarithmic factors. The third problem is the one defined in \citep{brown2021memorization}, for which we show a lower bound of $\Omega(d/2^{2n})$ for $n \lesssim \log d$, and we again match with a tight upper bound.

\paragraph{Problem setting and excess memorization.}
We now describe our approach and results in more detail starting with some more formal definitions and notation.
We study
binary classification
problems in which an algorithm $\Acal$ is given training data $X_{1:n} = X_1,\dots,X_n\overset{iid}{\sim}\Pcal_\theta$ sampled from some distribution $\Pcal_\theta$ over $\Xcal$ parameterized by some parameter $\theta$. To make our lower bounds stronger, following \citep{feldman2020does, brown2021memorization} we state them for average case problems, namely, we think of $\theta$ as itself being chosen from a known meta-distribution $\theta\sim \Psi$.
For a given data distribution $\Pcal_\theta$ and algorithm $\Acal$, we measure training data memorization of $\Acal$ by the mutual information between the model and the dataset drawn i.i.d. from $\Pcal_\theta$:
\begin{equation*}
\mem_n(\Acal,\Pcal_\theta)
:=I\left(\Acal(X_{1:n});X_{1:n}\right)~,
\end{equation*}
where $X_{1:n}\sim\Pcal_\theta^n$.
Notably, in this definition $\Pcal_\theta$ is fixed, and therefore $\mem_n(\Acal,\Pcal_\theta)$ does not count any information about the unknown data distribution $\theta$. 
To emphasize this property we refer to $\mem_n(\Acal,\Pcal_\theta)$ as {\em excess} data memorization. 
We further denote the average excess memorization for $\Acal$ on
an average-case
problem $\bP= (\Pcal_\theta)_{\theta\sim\Psi}$ by
\begin{align*}
\mathrm{mem}_n(\Acal,\bP)
:=\E_{\theta\sim\Psi}[\mem_n(\Acal,\Pcal_\theta)]
&=I\left(\Acal(X_{1:n});X_{1:n}\mid\theta\right)
\\&= I\left(\Acal(X_{1:n});X_{1:n}\right) - I\left(\Acal(X_{1:n});\theta\right)
~.
\end{align*}
Note, again, that conditioning on $\theta$ ensures that we are not counting the information that $\Acal$ learns about $\theta$ which is necessary for learning, but rather measuring excess memorization of the dataset.
Moreover, the latter equality (which holds by the chain rule for mutual information) provides an intuitive interpretation of this quantity: it is information the model has about the training data, after subtracting the ``relevant'' information about $\theta$, thus leaving the purely excess memorization.

\paragraph{Strong Data Processing Inequalities (SDPIs) imply memorization (Theorem~\ref{thm: excess mem}):}
We establish a direct connection between SDPIs and excess data memorization. We recall that the (regular) data processing inequality states that mutual information cannot increase as a result of post-processing, that is, in a Markov chain $A\rightarrow B\rightarrow C$, $I(A;C)\leq I(B;C)$. Strong data processing inequality holds when for a pair of jointly distributed random variables $(A,B)$, the step $A\rightarrow B$ necessarily reduces the mutual information by some factor $\rho<1$, referred to as the SDPI constant for $(A,B)$.  

In our context, for a randomly chosen $\theta\sim \Psi$, a dataset $X_{1:n}\sim\Pcal_\theta^n$, and an additional fresh test sample $X\sim\Pcal_\theta$, we have a Markov chain $X\rightarrow X_{1:n}\rightarrow\Acal(X_{1:n})$ since the only information $\Acal$ has about $X$ is through $X_{1:n}$. Thus we can deduce that $I(\Acal(X_{1:n}); X) \leq \rho \cdot I(\Acal(X_{1:n}); X_{1:n})$, where $\rho$ is the SDPI constant for $(X,X_{1:n})$. As is well-known, accurate binary prediction requires information, namely $I(\Acal(X_{1:n}); X) = \Omega(1)$ for any $\Acal$ with error that is $\leq1/3$. As a result of applying the SDPI, we get that $I(\Acal(X_{1:n}); X_{1:n}) = \Omega(1/\rho)$. 

As discussed, when $\theta$ is a random variable, $I(\Acal(X_{1:n}); X_{1:n})$ also counts the information that $\Acal(X_{1:n})$ contains about $\theta$. To obtain a lower bound on excess data memorization we need to subtract $I(\Acal(X_{1:n}); \theta)$ from $I(\Acal(X_{1:n}); X_{1:n})$. To achieve this, we consider the Markov chain $\theta \rightarrow X_{1:n}\rightarrow\Acal(X_{1:n})$ and denote by $\tau$ the SDPI constant of the pair
$(\theta, X_{1:n})$. Applying the SDPI to $\Acal(X_{1:n})$ and combining it with the lower bound on $I(\Acal(X_{1:n}); X_{1:n})$ then gives us the summary of the connection between SDPIs and memorization:
\begin{equation*}
\mem_n(\Acal,\bP) =\Omega\left(\frac{1-\tau_n}{\rho_n} \right)~,
\end{equation*}
where we emphasize the fact $\tau$ and $\rho$ depend on $n$.
We remark, that for $n=1$,
the first SDPI
coefficient $\rho_1$
is related to the proof technique of \citep{brown2021memorization} who reduce their learning problem to a variant of the so-called Gap-Hamming communication problem and then give an SDPI-based lower bound adapted from \citep{hadar2019communication}.

\paragraph{Approximate SDPIs via dominating variables (Theorem~\ref{thm: Z SDPI}):}
Our framework reduces memorization to computation of SDPI constants for pairs $(X_{1:n}, X)$ and $(X_{1:n}, \theta)$. However, known SDPIs deal primarily with individual samples from several very specific distributions (most notably, Bernoulli and Gaussian), and not datasets that appear to be much more challenging to analyze directly. We bypass this difficulty via a notion of a \emph{dominating} random variable for the dataset.
Specifically, if there exists a random variable $Z_\theta^\train$ and post-processing $\Phi$ such that $\Phi(Z_\theta^\train)\approx X_{1:n}$
(as distributions), then it suffices to prove SDPIs for the pair $(Z_\theta^\train,X)$. For our applications, it is crucial to allow $\Phi(Z_\theta^\train)$ to approximate $X_{1:n}$ and thus
our
reduction
incorporates the effects of approximation error in both SDPIs.
Our reduction
also allows using a dominating variable $Z_\theta^\test$ for the test point $X$, but in our applications $X$ is simple enough and this step is not needed.

\paragraph{Applications: memorization trade-offs and matching upper bounds (Theorem~\ref{thm: gaussian lb}, Theorem~\ref{thm: boolean lb} and Theorem~\ref{thm: sparse boolean lb}):}
We apply the techniques we developed to demonstrate that several natural learning problems 
exhibit smooth trade-offs between excess memorization and sample size. 
The first problem we consider is Gaussian clustering. In this problem, negative examples are sampled from $\Ncal(0_d,I_d)$, while positive examples are sampled from $\Ncal(\lambda \theta,(1-\lambda^2)I_d)$ for some scale $\lambda$. Our lower bounds are for $\theta\sim \Ncal(0_d,I_d)$. We pick $\lambda = \Theta(1/d^{1/4})$, ensuring that accurate learning is possible with just a single positive sample (by using either nearest neighbor or linear classifier). At the same time, our analysis demonstrates that any learning algorithm $\Acal$ for this problem that achieves non-trivial error satisfies
\[
\mem_n(\Acal,\bP) =\Omega\left(\frac{d}{n}\right)~.
\]

Moreover, for this problem we show that a matching (up to log factors) upper bound can be achieved by simple learning algorithms whenever $n\leq \sqrt{d}$. Hence, we overall establish that memorization can be reduced by using more data, and that is the only way to it (while maintaining accuracy).

We next consider a Boolean clustering problem. In this problem $\theta \in \{\pm1\}^d$, negative examples are sampled from the uniform distribution over $\{\pm1\}^d$, whereas positive examples are sampled from a product distribution over $\{\pm1\}^d$ with mean $\lambda \cdot  \theta$. Thus for $\theta$ chosen uniformly from $\{\pm1\}^d$, a random positive example from $\Pcal_\theta$ is coordinate-wise $\lambda$-correlated with $\theta$.
The problem can be thought of as a Boolean analogue of the Gaussian setting above, and we obtain nearly tight (up to a log factor) upper and lower bounds for this problem. The bounds are similar to those in the Gaussian case for $n\leq \sqrt{d}$ but almost no memorization is needed when $n\geq \sqrt{d} \log d$. We note however that the analysis of the approximately dominating variable is more involved in this case and is derived from composition results in differential privacy.

Finally, we consider a sparse Boolean hypercube clustering problem introduced in \citep{brown2021memorization}. In this setting, the data distribution is defined by a pair $\theta=(S,Y)$, where $S\subseteq [d]$ is a subset of coordinates and $Y \in \{\pm1\}^{S}$ are the values assigned to these coordinates. In the data distribution $\Pcal_\theta$ negative examples are uniform over $\{\pm1\}^d$, whereas the positive examples are uniform over the hypercube of all the points $x$ whose values in coordinates in $S$ are exactly $Y$: namely $x \in \{\pm1\}^d$ that satisfy the conjunction $\bigwedge_{i\in S} (x_i=Y_i)$. To sample $\theta$ we include each index in $S$ with probability $\approx 1/\sqrt{d}$ independently at random, and then assign to each coordinate a uniformly random value in $Y \in \{\pm1\}^{S}$. As noted in \citep{brown2021memorization}, when $n=1$ this problem is identical to the Boolean clustering problem defined above. However, for larger $n$, we show it requires much less memorization. Specifically, we give a lower and upper bound (tight up to log factors) that are:
\[
\mem_n(\Acal,\bP)=\tilde\Theta\left(\frac{d}{2^{2n}}\right)~.
\]

We remark that in these example applications, we chose parameters 
so as
to ensure that each problem is learnable from a single positive example, but requires $\Omega(d)$ bits to be memorized to do so. Beyond this extreme case, our techniques easily extend to show lower bounds in terms of correlation between samples, which also determines the smallest $n$ and $d$ that would be required for the learnability.

\subsection{From cluster classification to LLMs}
\label{sec:motivation}
We now briefly discuss how lower bounds for the simple classification problems we consider are related to the data memorization by LLMs. We first note that while LLMs are often used as generative models, underlying the sampler is a (soft) predictor of the next token given the preceding context. Thus an LLM is also a multiclass classifier. Second, LLMs (and many other ML models) either explicitly or implicitly rely on semantic data embeddings of the context, that is, embeddings in which semantically similar contexts are mapped to points close in Euclidean distance (or cosine similarity). In particular, nearby points are typically classified the same. As a result, when viewed in the embedding space, natural data distributions correspond to mixtures of (somewhat-disjoint) clusters of data points where points in the same cluster typically have the same label (cf. \citealp{reif2019visualizing,cai2021isotropy,radford2021learning}).

As has been widely observed (and discussed in \citep{feldman2020does}), for natural data distributions in many domains, the frequencies of these clusters tend to be long-tailed with a significant fraction of the entire data distribution being in low-frequency clusters. Such low frequency clusters have only few representatives in the training dataset (possibly just one). Accurate classification of a point from a low-frequency cluster requires being able to classify whether a test point belongs to the same cluster based on just a few examples of that cluster.

This shows that classifying points as belonging to some cluster or not is a subproblem that arises when learning from natural data. This raises the question of how to model such ``clusters''. While in practice cluster distributions will depend strongly on the representations used and may not have a simple form, one prototypical and widely studied choice is the Gaussian distribution. This distribution is known to be prevalent in natural phenomena (earning it the name ``normal''). The ubiquity of mixtures of Gaussian-like distributions is also the reason for the utility of techniques such as Gaussian Mixture Models for distribution modeling (cf. \citealp{reynolds2009gaussian}).

Putting these insights together, we get to the key application of our techniques: the problem of distinguishing a point from a Gaussian distribution from some null distribution given few samples, as a subproblem when the data distribution is a long-tailed mixture of Gaussian clusters.

\paragraph{Lower bounds for mixtures of clusters (Theorem~\ref{thm: accuracy to memorization}):}
Finally, in Section~\ref{sec: multiclass} we discuss how our techniques can be extended to a more detailed mixture-of-clusters model of data. Specifically, we consider data models based on a prior distribution over frequencies of clusters, as studied in \citep{feldman2020does}. The lower bound by \citet{brown2021memorization} is given in this model, but only clusters from which a single example was observed contribute to the memorization lower bound. We show that our more general lower bound approach extends to this mixture-of-clusters setting, with each cluster contributing to the total memorization lower bound according to the number of examples of that cluster in the training dataset.
In particular, our results demonstrate a smooth trade-off in which clusters with less representatives in the training data contribute more to the excess memorization of the learned model.

\subsection{Related Work}
\label{sec:related}
A fundamental theme in learning theory is that ``simple'' learning rules generalize \citep{blumer1987occam}.
In particular, there is a long line of work studying generalization bounds which provide various formalizations of the intuition that learners who use little information about their dataset must generalize.
Classical such notions include compression schemes \citep{littlestone1986relating} and the PAC-Bayes framework \citep{mcallester1998some}.
This theme is also the basis for the more recent use of mutual information (MI) between the dataset and the output of the algorithm to derive generalization bounds. The approach was first proposed in the context of adaptive data analysis by \citet{dwork2015generalization}, who used max-information to derive high-probability generalization bounds.
Building on this approach, \citet{russo2019much} proposed using the classical notion of MI to derive (in expectation) generalization bounds, with numerous subsequent works strengthening and applying their results \citep{raginsky2016information,xu2017information,feldman2018calibrating,bu2020tightening}. More recent developments in this line of work rely on the notion of {\em conditional mutual information} (CMI). Here, the conditioning is over a ghost sample which is different from the conditioning over the data distribution (i.e. $\theta$) considered in this work \citep{steinke2020reasoning}. The CMI roughly measures \emph{identifiability} of the samples in the dataset given the model. It is closely related to membership inference attacks
\citep{shokri2017membership,carlini2022membership,attias2024information}.

To demonstrate limitations of generalization methods based on MI, \citet{livni2024information} considers the setting of stochastic convex optimization (SCO). In this setting, he proves a lower bound on the MI of learners achieving asymptotically optimal error, which scales as $d/n^C$ for some constant $C$. While superficially this lower bound is similar to our result, the goal of the problem therein is to estimate an unknown $d$-dimensional parameter.
Moreover,
the coordinates of this parameter are chosen independently and thus the estimation problem requires effectively solving $d$ independent one-dimensional problems. This is in contrast to our setting, in which the problem is binary classification, and we make no assumptions on the representation of the model. We also remark that the trade-off in \citep{livni2024information} appears to be mostly an artifact of the proof, with natural algorithms achieving
nearly-optimal rates requiring excess memorization of $\Omega(d)$ bits for any $n$. \citet{attias2024information} recently proved lower bound on CMI in the SCO setting demonstrating that CMI cannot be used to recover known generalization bounds for SCO. By the nature of the definition, the CMI is at most $n$ and thus the lower bounds of the CMI are incomparable to the trade-offs in our work. 

Several works study lower bounds on the MI in the context of distribution-independent PAC learning of threshold functions establishing lower bounds which, at best, only scale logarithmically with the
description length of examples (which in our instances corresponds to the dimension) \citep{bassily2018learners,nachum2018direct,livni2020limitation}.

Lower bounds on memorization are also directly implied by lower bounds on differentially private algorithms. This follows from well-known results showing that the output of any $(\eps,\delta)$ differentially private algorithm $\Acal$ has bounded mutual information with the dataset \citep{mcgregor2010limits,bassily2018learners,feldman2018calibrating,duchi2019lower}. Specifically, for any distribution $P$ over $\Xcal$ and $X_{1:n}\sim P^n$, $I(\Acal(X_{1:n}),X_{1:n} ) = O(\eps^2 n + \delta n (H(P)+\log(1/\delta)))$ (see \citep{brown2021memorization} for an outline of a special case of this statement). In particular, our lower bounds on memorization directly imply lower bounds on accurate differentially private algorithms for solving the learning problems we defined.

The literature on phenomena related to memorization relies on a large variety of mostly informal notions. In the context of data extraction attacks, the definitions rely on the success of specific attacks that either feed a partial prompt \citep{carlini2021extracting} or examine the relative likelihood of training data under the model \citep{carlini2019secret}. Such definitions are useful for analyzing the success of specific attacks but are sensitive to the learning algorithms. In particular, minor changes to the algorithm can greatly affect the measures of memorization. They also do not distinguish between memorization of data relevant to the learning problem (e.g. memorization of capitals of countries in the context of answering geographic queries) from the irrelevant one.

Another related class of definitions considers memorization resulting from fitting of noisy data points (e.g.~\citealp{zhang2017rethinking}). Such memorization is referred to as {\em label} memorization and does not, in general, require memorization of data points that we study here. The known formal definition is not information-theoretic but rather directly examines the influence of the data point on the label \citep{feldman2020does}. At the same time, both label memorization and excess data memorization appear to be artifacts of learning from long-tailed data distributions.

On a technical side, SDPIs have a number of important applications in machine learning (and beyond), most notably, in the context of privacy preserving and/or distributed learning and estimation. We refer the reader to the book of \cite{polyanskiy2024information} for a detailed overview. These applications of SDPIs are not directly related to our use.

\section{Formal Problem Setting}\label{sec: formal setting}

\paragraph{Notation.}
We abbreviate a sequence $X_1,\dots,X_n$ by $X_{1:n}$.
$I_d$ denotes the $d\times d$ identity matrix, $0_d$ denotes the $d$-dimensional zero vector, and we occasional omit the subscript when the dimension is clear from the context.
Given a finite set $S$, we denote by $\Delta(S)$ the set of all distributions over $S$,
and by $\Ucal(S)$ the uniform distribution over $S$.
We denote $X\indep Y$ when $X,Y$ are independent random variables.
$\|\cdot\|$ denotes the Euclidean norm, and $d_{\TV}(\cdot,\cdot)$ denotes the total variation distance (which when applied to random variables, is the distance between their corresponding distributions).
For $x\in\{\pm1\}^d$, $\BSC_{p}(x)$ denotes the product distribution over $\{\pm1\}^d$ in which for $Y\sim \BSC_p(x)$, $\Pr[Y_i \ne x_i] =1-\Pr[Y_i=x_i]= p$ independently for every $i\in[d]$. The mapping from $x$ to $Y$ is usually referred to as the {\em binary symmetric channel (BSC)} in the context of data processing inequalities. Given random variables $X,Y,Z$, we denote by $H(X)$ the entropy of $X$,\footnote{We slightly abuse information theoretic notation by using it both for discrete and continuous random variables. In the latter case, definitions are with respect to the differential entropy.} by $I(X;Y):=H(X)-H(X\,|\, Y)$ the mutual information between $X$ and $Y$, and by $I(X;Y\,|\,Z):=\E_Z[I(X \,|\, Z;Y\,|\, Z)]$ the conditional mutual information.

\paragraph{Formal setting.}
We now formalize the problem setting we consider throughout this work, outlined in the introduction. 
The learning algorithm's goal is binary classification of a point drawn from a mixture distribution in which
with probability $1/2$ a (positive) point is drawn from the parameter dependent ``cluster'' distribution $\Pcal_\theta$ over $\Xcal$,
and with probability $1/2$ a (negative) point is drawn from a fixed ``null'' distribution $\Pcal_0$. 
Formally, for a problem parameter $\theta$, let $(X^\test,Y^\test)\sim\Pcal_\theta^{\test}$ be the mixture distribution over $\Xcal\times\{0,1\}$ such that with probability $\half:~X^\test\sim\Pcal_\theta$ and $Y^\test=1$; otherwise with probability $\half:~X^\test\sim\Pcal_0$ and $Y^\test=0$.
The algorithm returns a classifier $h:\Xcal\to\{0,1\}$,
aiming at minimizing the classification error:
\[
\err(h):=\Pr\displaystyle_{
\substack{(X^\test,Y^\test)\sim\Pcal_\theta^{\test}}
}
\left[h(X^\test)\neq Y^\test\right]
~.
\]

Note that $\Pcal_0$ is fixed and therefore negative examples do not carry any information about the learning problem (and can be generated by the learning algorithm itself). Therefore, without loss of generality, we can assume that the training dataset given to a learning algorithm $\Acal$ consists of $n$ positive data points  $X_1,\dots,X_n\overset{iid}{\sim}\Pcal_\theta$.

In our problems, the parameter $\theta$ will be sampled from a known meta-distribution $\theta\sim\Psi$ and we denote the resulting (average-case) learning problem by $\bP=(\Pcal_\theta)_{\theta\sim\Psi}$. Further, the null distribution $\Pcal_0$ will always be equal to the marginal distribution of $X\sim \Pcal_{\theta}$ for $\theta\sim\Psi$ and thus we do not specify it explicitly. The loss of a learning algorithm $\Acal$ on the problem $\bP=(\Pcal_\theta)_{\theta\sim\Psi}$ is defined accordingly as
\[
\err(\Acal,\bP):=\E_{\substack{
\theta\sim\Psi
\\X_1,\dots,X_n\sim\Pcal_\theta
\\h\leftarrow\Acal(X_{1:n})}}
\left[
\err(h)
\right]
~.
\]

We also recall the definition of excess data memorization for an algorithm $\Acal$ and data distribution $\Pcal_\theta$ \citep{bassily2018learners,brown2021memorization}:
\begin{equation*}
\mem_n(\Acal,\Pcal_\theta)
:=I\left(\Acal(X_{1:n});X_{1:n}\right)~,
\end{equation*}
where $X_{1:n}\sim\Pcal_\theta^n$.
For an average case problem $\bP= (\Pcal_\theta)_{\theta\sim\Psi}$ excess memorization for $\Acal$ is defined as 
\begin{align*}
\mem_n(\Acal,\bP)
:=\E_{\theta\sim\Psi}[\mem_n(\Acal,\Pcal_\theta)]
&=I\left(\Acal(X_{1:n});X_{1:n}\mid\theta\right)
~.
\end{align*}
We also denote the minimal (i.e., necessary)
memorization for algorithms with error of at most $\alpha$ by 
\[
\mem_n(\bP,\alpha) := \inf_{\Acal: \,\err(\Acal,\bP)\leq \alpha}  \mem_n(\Acal,\bP)~,
\]
and let $\mem_n(\bP) := \mem_n(\bP, 1/3)$.

\section{General Framework: SDPIs and Memorization}
\label{sec: SDPI and MEM general}

In this section, we will introduce the main machinery that allows us to derive excess memorization lower bounds via \emph{strong data processing inequalities} (SDPIs).
We start by recalling the definition of an SDPI. Specifically, we will consider so called {\em input-dependent} SDPIs in which the marginal distribution of the source random variable is fixed \citep{polyanskiy2024information}.

\begin{definition}
Given a pair of jointly distributed random variables $(A,B)$, we say that
$A,B$ satisfy $\rho$-SDPI if for any $M$ such that $A\indep M\,|\,B:~I(M;A)\leq \rho I(M;B)$.
\end{definition}

Recall that all random variables satisfy the definition above for $\rho=1$, which is simply the ``regular'' data processing inequality (DPI).
A \emph{strong} DPI refers to the case $\rho<1$. This means that any $M$ which is a post-processing of $B$ (i.e., $A\indep M\,|\,B$), has \emph{strictly} less information about $A$ than it does about $B$. Equivalently, any post-processing of $B$ must have $\rho^{-1}$-times \emph{more} information about $B$. This observation will serve as the basis of our results, and accordingly, we will aim to prove SDPIs with a small SDPI constant $\rho$.

SDPIs constitute a fundamental concept in information theory dating back to \citet{ahlswede1976spreading}, and their study remains an active area of research
\citep{raginsky2016strong,polyanskiy2017strong,polyanskiy2024information}.
Here we recall two canonical examples of SDPI in which the coefficient $\rho$ results from weak correlation between the marginals.

\begin{fact}
[\citealp{polyanskiy2024information}, Example 33.7+Proposition 33.11]
\label{fact: classic SDPI gauss}
Suppose $(A,B)$ is a $2d$-dimensional Gaussian distribution such that
marginally $A,B\sim\Ncal(0_d,I_d)$ and for each coordinate $i\in[d]:~\E[A_i\cdot B_i]=\sqrt{\rho}$. Then
$A,B$ satisfy the $\rho$-SDPI.
\end{fact}

\begin{fact}
[\citealp{polyanskiy2024information}, Example 33.2]
\label{fact: classic SDPI bool}
Suppose $(A,B)$ are $\sqrt{\rho}$-correlated uniform Boolean vectors, namely $A\sim\Ucal(\{\pm1\}^d),~B=\BSC_{\frac{1-\sqrt{\rho}}{2}}(A)$.
Then $A,B$ satisfy the $\rho$-SDPI.
\end{fact}

In some applications, using ``approximate'' SDPIs appears to be crucial for achieving meaningful lower bounds, so we introduce the following approximate version of this notion.

\begin{definition}
Given a pair of jointly distributed random variables $(A,B)$, we say that
$A,B$ satisfy
$(\rho,\delta)$-approximate-SDPI, or simply $(\rho,\delta)$-SDPI, if
they can be coupled with a random variable $\tilde{B}$ such that: (1) $A$ and $\tilde{B}$ satisfy $\rho$-SDPI, and (2) for every $a\in \mathrm{supp}(A)$, we have $d_{\TV}(\tilde{B}\mid A = a, B \mid A = a ) \le \delta$.
\end{definition}

Clearly, the definition above reduces to the standard notion of SDPI when $\delta=0$, with $\tilde{B}=B$.
Its main utility is through the following result, which quantifies how the $\delta$-approximation results in an additive factor to the standard SDPI setting.

\begin{proposition} \label{prop: approximate sdpi}
    Suppose $(A,B)$ are joint random variables satisfying $(\rho,\delta)$-SDPI.
    Let $\Acal:\mathcal{B}\to\mathcal{M}$ be a (possibly randomized) post-processing of $B$, where $\mathcal{B}$ denotes the support of $B$. Then
    \[
    I(\Acal(B);A)\leq \rho I(\Acal(B);B)+
    8\delta\log(|\Mcal|/\delta)~.
    \]
\end{proposition}

\begin{remark}
Throughout the paper, we use the convention $\delta\log(|\Mcal|/\delta)=0$ when $\delta=0,|\Mcal|=\infty$.
\end{remark}

The proof of Proposition~\ref{prop: approximate sdpi} is provided in Appendix~\ref{sec: proofs framework}.
The rest of this section is structured as follows. In Section~\ref{sec: sdpi to mem} we show
how to bound the quantity of interest $\mem_n(\Acal,\bP)$ whenever certain SDPIs are present in the learning problem.
Subsequently, in Section~\ref{sec: sdpi reduction}, we will address the issue of when we should expect the required SDPIs to hold, and how to compute their corresponding coefficients via an approximate reduction.

\subsection{SDPIs imply Memorization} \label{sec: sdpi to mem}

We now present the main result that relates excess memorization to SDPIs.
Recalling that the problem definition involves the distributions $\Pcal_\theta$ and the parameter distribution $\Psi$, the next theorem shows that memorization is necessary whenever certain SDPIs hold in the learning problem.

\begin{theorem} \label{thm: excess mem}
    Let $\bP=(\Pcal_\theta)_{\theta\sim\Psi}$ be a learning problem satisfying the following:
    \begin{enumerate}
    \item (Data generation SDPI) The variables $(\theta, X_{1:n})$ for $\theta \sim\Psi$ and $X_{1:n}\sim\Pcal_\theta^n$ satisfy $(\tau_n,\epsilon_n)$-SDPI.
    \item (Test/train SDPI) The variables $(X,X_{1:n})$ for $\theta \sim\Psi$, 
    $X\sim\Pcal_\theta$ and $X_{1:n} \sim\Pcal_\theta^n$ satisfy $(\rho_n,\delta_n)$-SDPI.
    \end{enumerate}
Then any algorithm $\Acal:\Xcal^n\to\Mcal$ for $\bP$ satisfies the excess memorization bound: 
\[
\mathrm{mem}_n(\Acal,\bP)
\geq \frac{1-\tau_n}{\rho_n}I(\Acal(X_{1:n});X) -
\mathrm{neg}_n
~,
\]
where $\mathrm{neg}_n:=8\delta_n\frac{1-\tau_n}{\rho_n}\log(|\Mcal|/\delta_n)+8\epsilon_n\log(|\Mcal|/\epsilon_n)$.
Moreover, for any $\alpha<\half:$
\[
\mem_n(\bP,\alpha)
\geq \frac{1-\tau_n}{\rho_n} C_{\alpha} -\mathrm{neg}_n
~,
\text{~~~where~~~}C_{\alpha} := (1-2\alpha)\log\left(\frac{1-\alpha}{\alpha}\right)~.
\]
\end{theorem}

\begin{figure}
\begin{center}
	\includegraphics[trim=0cm 12.5cm 0cm 12.5cm,clip=true, width=0.8\linewidth]{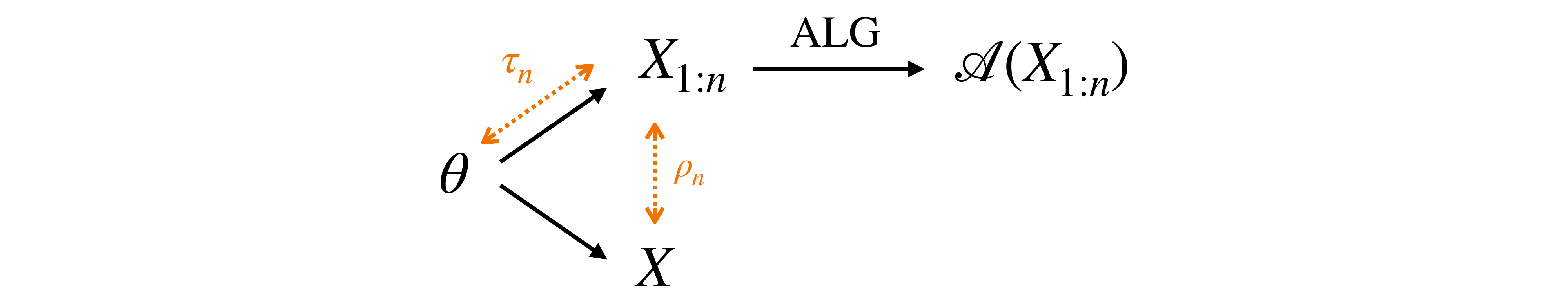}%
	\caption{Illustration for Theorem~\ref{thm: excess mem}.
     Given $\theta\sim\Psi,\,X_{1:n}\sim\Pcal_\theta^n,\,X \sim\Pcal_\theta$, 
     the orange arrows represent two SDPIs, which together necessitate excess memorization on the order of $(1-\tau_n)/\rho_n$.
    }
	\label{fig:SDPI_flow1}
	\end{center}
\end{figure}

\begin{proof}[Proof of Theorem~\ref{thm: excess mem}]
Note that the data generation SDPI implies by Proposition~\ref{prop: approximate sdpi} that
\begin{align} \label{eq: data gen approx sdpi}
\tag{data gen SDPI}
I(\Acal(X_{1:n});\theta)
\leq \tau_n I(\Acal(X_{1:n});X_{1:n})+8\epsilon_n\log(|\Mcal|/\epsilon_n)~.
\end{align}
Similarly, the test/train SDPI implies by Proposition~\ref{prop: approximate sdpi}, after rearrangement, that
\begin{align} \label{eq: test/train approx sdpi}
\tag{test/train SDPI}
I(\Acal(X_{1:n});X_{1:n})\geq
\frac{1}{\rho_n}I(\Acal(X_{1:n});X)-8(\delta_n/\rho_n)\log(|\Mcal|/\delta_n)~.
\end{align}
We therefore see that
\begin{alignat*}{2}
\mathrm{mem}_n(\Acal,\bP)
&=I(\Acal(X_{1:n});X_{1:n}\mid\theta) \qquad &~
\\&=I(\Acal(X_{1:n});X_{1:n},\theta) 
-I(\Acal(X_{1:n});\theta) \qquad &{\color{darkgray}\left[\text{chain rule}\right]}
\\
&=I(\Acal(X_{1:n});X_{1:n})-I(\Acal(X_{1:n});\theta)
\qquad &{\color{darkgray}\left[\Acal(X_{1:n})\indep\theta\,|\,X_{1:n}\right]}
\\
&\geq (1-\tau_n)I(\Acal(X_{1:n});X_{1:n}) -8\epsilon_n\log(|\Mcal|/\epsilon_n)
\qquad &{\color{darkgray}(\text{\ref{eq: data gen approx sdpi}})}
\\
&\geq (1-\tau_n) \left(\frac{1}{\rho_n}\cdot I(\Acal(X_{1:n});X)-(8\delta_n/\rho_n)\log(|\Mcal|/\delta_n)\right)
\qquad &{\color{darkgray}(\text{\ref{eq: test/train approx sdpi}})}
\\&~~~~~
-8\epsilon_n\log(|\Mcal|/\epsilon_n)
~. 
\end{alignat*}
This establishes the first claim.
To further prove the second claim, it remains to show that if
$\err(\Acal, \bP) \le \alpha$ then $I(\Acal(X_{1:n}) ; X) \ge C_\alpha$, namely that a non-trivial error bound implies a mutual information lower bound.
This Fano-type argument is rather standard, and we defer its proof to Appendix~\ref{sec: complete excess mem proof}.
\end{proof}

Before continuing, we discuss the typical use of Theorem~\ref{thm: excess mem}.
Our aim is to show that problems of interest satisfy SDPI with $\tau_n,\rho_n\ll 1$, and to quantify these SDPI constants as a function of $n$.
In our results $\epsilon_n,\delta_n$ will be negligible, resulting in a negligible term $\mathrm{neg}_n$. 
Temporarily ignoring this negligible additive term,
as long as $\tau_n\leq 1/2$ we see that Theorem~\ref{thm: excess mem} implies for any learning algorithm $\Acal:~\mathrm{mem}_n(\Acal,\bP)=\Omega(1/\rho_n)$.
In some of the applications, we will see that $\rho_n = \Theta(n\rho_1)$, which is closely related to ``advanced composition'' from the differential privacy \citep{DworkRV10-boosting}.
Intuitively, this means that the dataset $X_{1:n}$ becomes more correlated with $\theta$ and $X\sim \Pcal_\theta$ as the number of samples grows, hence having more information about test.
Finally, in the simplest setting of interest in which the
inner product
between two independent samples scales as $\sqrt{d}$, even a single sample suffices to achieve low (expected) classification error, yet
in this setting of parameters, $\mathrm{mem}_n(\Acal,\bP)
\gtrsim {1}/{\rho_n}
\gtrsim {1}/{n\rho_1}
\gtrsim {d}/{n}$.

Overall, 
Theorem~\ref{thm: excess mem} reduces proving memorization lower bounds to proving two SDPIs,
and quantifies memorization via the coefficients $\tau_n,\rho_n$.
We remark that the data generation SDPI can be bypassed whenever it is easy (or easier) to prove an explicit upper bound on $I(\Acal(X_{1:n});\theta)$ instead of relating it to $I(\Acal(X_{1:n});X_{1:n})$,
as we do
in our third application in Section \ref{sec: sparse bool}.
The next section addresses the question of computing these coefficients.

\subsection{Proving SDPIs via Dominating Variables} \label{sec: sdpi reduction}

Having established in Theorem~\ref{thm: excess mem} that memorization follows from SDPIs in the process that generates the data, we address the computation of the corresponding SDPI coefficients.
Our derivation of excess memorization lower bounds proceeds by reducing memorization to implicit variables that dominate the learning problem.
The following theorem formalizes this, as illustrated in Figure~\ref{fig:SDPI_flow2}.

\begin{figure}
\begin{center}
	\includegraphics[trim=0cm 10.5cm 0cm 10.5cm,clip=true, width=0.8\linewidth]{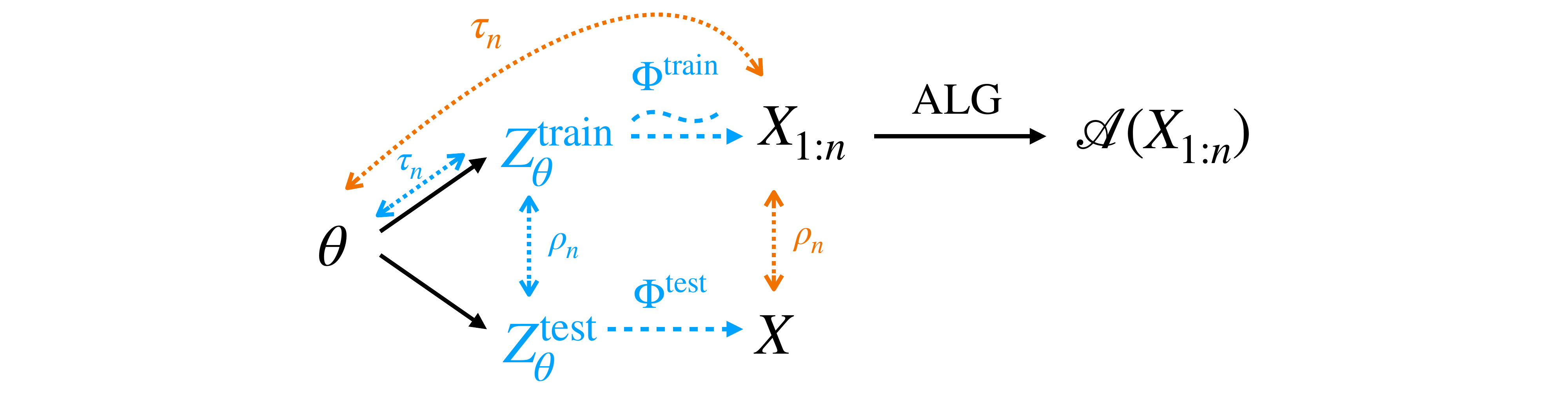}%
	\caption{Illustration of Theorem~\ref{thm: Z SDPI}. 
The blue variables
dominate the learning problem,
and each blue SDPI implies an (approximate) SDPI in orange, resulting in excess memorization.}
	\label{fig:SDPI_flow2}
	\end{center}
\end{figure}

\begin{theorem} \label{thm: Z SDPI}
Let $\bP=(\Pcal_\theta)_{\theta\sim\Psi}$ be a learning problem,
and suppose $(Z_\theta^{\train},Z^\test_\theta)$
are jointly distributed random variables
parameterized by $\theta$ so that
    $Z_\theta^{\train}\indep Z^\test_\theta\mid\theta$,
        and that there are mappings $\Phi^{\train},\Phi^\test$ such that for all $\theta$, $d_{\TV}(\Phi^{\train}(Z_\theta^{\train}),\Pcal_\theta^n)\leq \delta_n$,
        and $\Phi^\test(Z^\test_\theta)\sim\Pcal_\theta$.
Then:
\begin{enumerate}
    \item (Data generation SDPI) If the marginal pair $(\theta,Z_{\theta}^{\train})$ for $\theta\sim\Psi$
 satisfies $\tau_n$-SDPI, then $(\theta, X_{1:n})$ for $\theta \sim\Psi$, $X_{1:n}\sim\Pcal_\theta^n$ satisfy $(\tau_n,\delta_n)$-SDPI.
    \item (Test/train SDPI) If the marginal pair $(Z^\test_\theta,Z_\theta^{\train})$ satisfies $\rho_n$-SDPI, then $(X,X_{1:n})$ for $X\sim\Pcal_\theta$ and $X_{1:n}\sim\Pcal_\theta^n$ satisfies $(\rho_n,\delta_n)$-SDPI.
\end{enumerate}
\end{theorem}

\begin{proof}[Proof of Theorem~\ref{thm: Z SDPI}]
To prove the first item, recall that for all $\theta$ it holds that $d_{\TV}(\Phi^{\train}(Z_\theta^{\train}),\Pcal_\theta^n)\leq\delta$, so by definition it suffices to show that $(\theta,\Phi^{\train}(Z_\theta^{\train}))$ satisfy $\tau_n$-SDPI.
Indeed, for any $M$ such that $\theta\indep M\,|\,\Phi^{\train}(Z_\theta^{\train}):$

\begin{align*}
I(M;\theta)
\leq \tau_n I(M;Z_\theta^{\train})
\leq \tau_n I(M;\Phi^{\train}(Z_\theta^{\train}))~,
\end{align*}
where the first inequality follows from the data generation SDPI assumption,
and the second inequality follows from the DPI since $Z_{\theta}^{\train}\indep M\,|\,\Phi^{\train}(Z_\theta^{\train})$.

To prove the second item, we first note that it suffices to show that  $(\Phi^\test(Z^\test_\theta),\Phi^{\train}(Z_\theta^{\train}))$ satisfy $\rho_n$-SDPI. This is true since $\Phi^\test(Z^\test_\theta)$ is distributed as $X$ (according to $\Pcal_\theta$) and $Z_\theta^{\train}\indep Z^\test_\theta\mid\theta$. By our assumption, this means that conditioned on any value of $\Phi^\test(Z^\test_\theta)$, the distribution of $\Phi^{\train}(Z_\theta^{\train})$ is $\delta_n$ close in TV distance to $\Pcal_\theta^n$. Thus for every $\theta$ a pair $(X,X_{1:n})\sim \Pcal_\theta \times \Pcal_\theta^n$ can be seen as a sample from $\Phi^\test(Z^\test_\theta)$ and an independent sample from a distribution close in TV distance to the distribution of $\Phi^{\train}(Z_\theta^{\train}))$.

Now, for any $M$ such that $\Phi^{\test}(Z_{\theta}^{\test})\indep M\,|\,\Phi^{\train}(Z_\theta^{\train}):$
\begin{align*}
I(M;\Phi^{\test}(Z_{\theta}^{\test}))
\overset{\mathrm{(1)}}{\leq} I(M;Z_{\theta}^{\test})
\overset{\mathrm{(2)}}{\leq} \rho_n I(M;Z_{\theta}^{\train})
\overset{\mathrm{(3)}}{\leq} \rho_n I(M;\Phi^{\train}(Z_\theta^{\train}))~,
\end{align*}
where $(1)$ is the DPI, $(2)$ follows from the test/train SDPI assumption since $\Phi^{\train}(Z_\theta^{\train})\indep Z^\test_\theta\,|\,Z_\theta^{\train}$ and therefore $M\indep Z^\test_\theta\,|\,Z_\theta^{\train}$,
and $(3)$ follows from the DPI since $M\indep Z_\theta^{\train}\,|\,\Phi^{\train}(Z_\theta^{\train})$.

\end{proof}

The theorem above shows that if a pair of variables $Z_\theta^{\train},Z^\test_\theta$
simulate the train and test data (up to some approximation), 
then we can reduce the computation of the data generation coefficient $\tau_n$ and test/train coefficient $\rho_n$ to these dominating variables with only (presumably small) additive loss.

\begin{remark}
The dominating variables may appear related to the concept of ``sufficient statistics'' (cf. \citealp{cover1999elements,polyanskiy2024information}) of the sample and test data. The main difference is that the discussed variables need not be statistics of the data, i.e. computable from it.
\end{remark}

\section{Applications} \label{sec: app}

In this section, we describe several applications of our framework. Our focus here will be on  problems where high accuracy can be achieved given a single (positive) example. This setting is the closest to our motivating problem of memorization of entire data points. However, by appropriately choosing the parameters, the trade-off can be shown in problems with higher sample complexity.

\subsection{Gaussian Clustering}

We consider a Gaussian clustering problem. Formally, given $\lambda\in(0,1)$ to be fixed later,
the problem $\bP_{\mathrm{G}}=(\Pcal_\theta)_{\theta\sim\Psi}$ is defined as
\begin{align*}
\Pcal_\theta~=~\Ncal(\lambda \theta,(1-\lambda^2)I_d)~,
~~~~~
\theta\sim\Psi~=~\Ncal(0_d,I_d)~.
\end{align*}
Note that the null distribution, namely the marginal of $\Pcal_\theta$ over $\theta\sim\Psi$, equals $\Pcal_0=\Ncal(0_d,I_d)$.
Hence, the problem corresponds to classifying between samples that are $\lambda$-correlated with an unknown parameter $\theta$, and samples that have zero correlation with $\theta$.

Our main result for this problem instance is the following:

\begin{theorem} \label{thm: gaussian lb}
In the Gaussian clustering problem $\bP_{\mathrm{G}}$, assume $\lambda=C d^{-1/4}$ for some sufficiently large absolute constant $C>0$.
Then the following hold:
\begin{itemize}
    \item There exists an algorithm $\Acal$, that given a single sample (i.e. $n=1$) satisfies $\err(\Acal,\bP_{\mathrm{G}}) \leq 0.01$. 
    \item  For $\alpha < 1/2$,
       \[
\mem_n(\bP_{\mathrm{G}},\alpha)\geq 
\frac{1-\lambda^2}{\lambda^4 n}\cdot (1-2\alpha)\log\left(\frac{1-\alpha}{\alpha}\right)~ ,
    \] and, in particular, $\mem_n(\bP_{\mathrm{G}})=\Omega\left(\frac{d}{n}\right)$.
\item
The lower bound above is tight: There is a learning algorithm $\Acal$ such that $\err(\Acal,\bP_{\mathrm{G}}) \leq 0.01$  and $\mem_n(\Acal,\bP_{\mathrm{G}})=O(d/n)$.
\end{itemize}

\end{theorem}

\begin{remark} \label{rem: error log}
For our upper bounds, we focus on small constant error chosen as $0.01$ for simplicity. More generally, if the correlation is set $\lambda=Cd^{-1/4}$ for $C=\Theta(\sqrt{\log(1/\alpha)})$, the same statements hold for learners with error at most $\alpha$, affecting only logarithmic terms in the resulting excess memorization.
\end{remark}

We provide here a sketch of the proof of Theorem~\ref{thm: gaussian lb}, which appears in Appendix~\ref{sec: gaussian proof}.

\renewcommand*{\proofname}{\textbf{Proof sketch of Theorem~\ref{thm: gaussian lb}}}
\begin{proof}
The first item follows from standard Gaussian concentration bounds, since for a single sample $X_1\sim\Pcal_\theta=\Ncal(\lambda\theta,(1-\lambda^2)I_d)$, it holds that on one hand for
$X\sim\Pcal_\theta:~\inner{X,X_1}\gtrsim \lambda^2d= C^2\sqrt{d}$ with high probability, while on the other hand for
$X\sim\Pcal_0=\Ncal(0_d,I_d):~\inner{X,X_1}\lesssim \sqrt{d}$. Therefore, for a sufficiently large $C>0$, a linear classifier will have small error.

To prove the second item, we rely on the ideas presented in Section~\ref{sec: SDPI and MEM general}. Particularly,
we note that the information about $\theta$ contained in the dataset $X_1,\dots,X_n\sim\Ncal(\lambda\theta,(1-\lambda^2)I_d)$ is dominated by the empirical average $\hat{X}:=\frac{1}{n}\sum_{i\in[n]}X_i\sim \Ncal(\lambda\theta,\frac{1-\lambda^2}{n}I_d)$.
Noting that $\lambda\theta\sim\Ncal(0_d,\lambda^2 I_d)$, we see that the variance of $\hat{X}$ in every direction is $\lambda^2+\frac{1-\lambda^2}{n}= \frac{1+\lambda^2(n-1)}{n}$.
Therefore, by rescaling $\hat{X}$ by $\sqrt{n/(1+\lambda^2(n-1))}$, we get a Gaussian with unit variance 
$Z_\theta^{\train}= \sqrt{n/(1+\lambda^2(n-1))}\cdot\hat{X}$
which is a dominating variable for the dataset, whose coordinate-wise correlation with a fresh sample $X\sim\Ncal(\lambda\theta,(1-\lambda^2)I)$ is $\lambda^2 \sqrt{n/(1+\lambda^2(n-1))}$. By Fact~\ref{fact: classic SDPI gauss}, this gives us a test/train SDPI with $\rho_n=\lambda^4n/(1+\lambda^2(n-1))$. 
The same dominating variable also proves a data generation $\tau_n$-SDPI by computing the coordinate-wise correlation between $\theta$ and $Z_\theta^{\train}$ as $\lambda\sqrt{n/(1+\lambda^2(n-1))}$, and therefore $\tau_n=\lambda^2n/(1+\lambda^2(n-1))$ once again by Fact~\ref{fact: classic SDPI gauss}.
Plugging these SDPI coefficients into Theorem~\ref{thm: excess mem}, we see that $\mem_n(\LG)\gtrsim (1-\tau_n)/\rho_n
=\frac{1-\lambda^2n/(1+\lambda^2(n-1))}{\lambda^4n/(1+\lambda^2(n-1))}
=\frac{1-\lambda^2}{\lambda^4n}
\approx d/n$.

To prove the last item, we consider an algorithm that computes $\tilde{\mu}=\hat{X}+\eta$ where $\hat{X}\sim \Ncal(\lambda\theta,\frac{1-\lambda^2}{n}I_d)$
is the empirical average and $\eta\sim N(0_d,I_d)$ is an independent isotropic Gaussian, and returns a linear classifier with respect to $\tilde{\mu}$.
The idea is that Gaussian concentration arguments ensure that on one hand for
$X\sim\Pcal_\theta:\,\inner{X,\tilde{\mu}}\gtrsim \lambda^2\|\theta\|^2\gtrsim C^2\sqrt{d}$
with high probability,
while on the other hand for
$X\sim\Pcal_0:\,\inner{X,\tilde{\mu}}\lesssim C\sqrt{d}$, and so setting a threshold of $C^{3/2}\sqrt{d}$ achieves low error. To bound the algorithm's memorization, we analyze the conditional mutual information $I(\tilde{\mu},\hat{X}\mid\theta)= \E_{\theta\sim\Psi, u \sim \hat{X}}[D_{\mathrm{KL}}(u + \eta \,||\, \hat{X} + \eta)]$, and calculate that the latter is bounded by $O(d/n)$ using the formula for the KL divergence between Gaussians.

We remark that a previous version of this work provided a different algorithm, which exhibits a nearly-tight memorization upper bound (up to a log factor), yet it only applied when $n\lesssim \sqrt{d}$. This algorithm is based on a projection approach which we will also use for the Boolean application in the next section. For completeness, we include this algorithm as well as its analysis in the appendix as well.
\end{proof}
\renewcommand*{\proofname}{Proof}

\subsection{Boolean Clustering}

We consider a Boolean clustering problem, which is the Boolean analogue of the previously discussed Gaussian mean estimation problem. Formally, given $\lambda\in(0,1)$ to be chosen later, the problem $\bP_{\mathrm{B}}=(\Pcal_\theta)_{\theta\sim\Psi}$ corresponds to
\begin{align*}
\Pcal_\theta~=~\BSC_{\frac{1-\lambda}{2}}(\theta)~,
~~~~~
\theta\sim\Psi~=~\Ucal(\{\pm1\}^d)~.
\end{align*}
Note that the null distribution is uniform $\Pcal_0=\Ucal(\{\pm1\}^d)$.
Namely, given samples correlated with $\theta$ (which can be thought of as cluster around $\theta$),
the problem corresponds to classifying between fresh samples that are $\lambda$-correlated coordinate-wise with some $\theta$, and uniformly generated samples.
Our main result for this problem instance is the following:

\begin{theorem} \label{thm: boolean lb}
In the Boolean clustering problem $\bP_{\mathrm{B}}$, assume $\lambda=C d^{-1/4}$ for some sufficiently large absolute constant $C>0$.
Then the following hold:
\begin{itemize}
    \item There exists an algorithm $\Acal$, that given a single sample (i.e. $n=1$) satisfies $\err(\Acal,\bP_{\mathrm{B}}) \leq 0.01$. 
        \item For $\alpha<\half$,
\[
\mathrm{mem}_n(\bP_{\mathrm{B}},\alpha)
=\Omega\left(\frac{1-8\lambda^2n\log(dn\log|\Mcal|/C_\alpha)}{\lambda^4n\log(dn\log|\Mcal|/C_\alpha)}C_{\alpha}\right)~,
~~~C_{\alpha} := (1-2\alpha)\log\left(\frac{1-\alpha}{\alpha}\right)~.
\]
In particular, as long as
$n\leq c\sqrt{d}/\log(d\log(|\Mcal|)/C_\alpha)$,
for some sufficiently small absolute constant $c>0$,
and 
$|\Mcal|\leq \exp(d^{\tilde{C}})$ for any absolute constant $\tilde{C}>0$ fixed in advance, then
\[
\mathrm{mem}_n(\bP_{\mathrm{B}},\alpha)=
\Omega\left(\frac{1-8\lambda^2n\log(dn/C_\alpha)}{\lambda^4n\log(dn/C_\alpha)}C_{\alpha}\right)~,
\]
and $\mathrm{mem}_n(\bP_{\mathrm{B}})=\widetilde{\Omega}\left(\frac{d}{n}\right)$.

\item The lower bound above is nearly-tight: On one hand, if $n\leq\sqrt{d}$ then there is a learning algorithm $\Acal$ such that $\err(\Acal,\bP_{\mathrm{B}}) \leq 0.01$  and $\mem_n(\Acal,\bP_{\mathrm{B}})=O(d/n)$.
On the other hand, there is a learning algorithm $\Acal$ such that $\err(\Acal,\bP_{\mathrm{B}}) \leq 0.01$  and $\mem_n(\Acal,\bP_{\mathrm{B}})=O(d^2n\exp(-n/\sqrt{d}))$, and so if $n\gtrsim \sqrt{d}\log (d)$ then $\mem_n(\Acal,\bP_{\mathrm{B}})\leq 1/\mathrm{poly}(d)$.

\end{itemize}

\end{theorem}

The theorem shows that up to $n\approx\sqrt{d}$ samples the excess memorization decays as $\tilde{\Theta}(d/n)$, and that afterwards it drops to nearly zero.
We note that the extremely mild condition that $|\Mcal|\in\exp(\mathrm{poly}(d))$, namely the hypothesis class size is not super-exponential in the dimension, is likely just an artifact of the proof technique which is based on the approximation strategy introduced in Section~\ref{sec: sdpi reduction}.

The proof of Theorem~\ref{thm: boolean lb} is similar in spirit to that of Theorem~\ref{thm: gaussian lb} as we previously sketched,
yet the technical details in the proof of the memorization lower bound (second item) are more challenging.
This follows from the fact that, as opposed to the Gaussian case, the boolean dataset $X_{1:n}\sim \BSC_{\frac{1-\lambda}{2}}(\theta)^n$ does not have a simple dominating variable.
Instead, we use arguments related to advanced composition from the differential privacy literature, to argue that $X_{1:n}$ is statistically close to a post-processing of $Z_\theta^{\train}\sim\BSC_{\frac{1-\xi}{2}}(\theta)$ for $\xi\approx\sqrt{n}\lambda$, namely a single variable which is $\sqrt{n}$-times more correlated with $\theta$.
We can then invoke Fact~\ref{fact: classic SDPI bool} and use our approximate reduction in Theorem~\ref{thm: Z SDPI} to obtain the required SDPIs, by noting that the coordinate-wise correlation of $\theta$ and $Z_\theta^{\train}$ is $\sqrt{\tau_n}=\xi$, whereas that of $Z_\theta^{\train}$ and $X\sim\BSC_{\frac{1-\lambda}{2}}(\theta)$ is $\sqrt{\rho_n}=\xi\lambda$. Consequently, up to logarithmic factors,
we obtain $\mem_n(\bP_{\mathrm{B}})\gtrsim\frac{1-\xi^2}{\lambda^2\xi^2}\approx\frac{1-\lambda^2n}{\lambda^4n}$, and under the assignment of $\lambda$ and assumption that $n\ll\sqrt{d}$, the latter simplifies to $\widetilde{\Omega}(d/n)$.

The nearly-matching upper bounds are realized by an algorithm that computes the bit-wise majority vote over the sample. For the regime $n\lesssim \sqrt{d}$, the algorithm only computes the majority along the first $\ell\approx d/n$ coordinates, and returns a linear classifier in the projected space. Concentration arguments ensure that the algorithm has small error, while clearly requiring at most $\ell$ bits of memory, thus in particular no more than $\ell\approx d/n$ bits from the training set can be memorized.
When $n\gtrsim \sqrt{d}\log(d)$, computing the majority vote in each coordinate reconstructs the parameter $\theta$ with very high confidence, so in this regime, accurate learning is possible with nearly zero excess memorization.
The full proof appears in Appendix~\ref{sec: boolean proof}.

\subsection{Sparse Boolean Hypercube} \label{sec: sparse bool}
Finally, we apply our framework to the  sparse Boolean hypercube clustering problem defined by \citet{brown2021memorization}. Given $\nu>0$ to be chosen later, the problem $\bP_{\mathrm{sB}}=(\Pcal_\theta)_{\theta\sim\Psi}$ is defined as follows. The parameter $\theta=(S,y)$ is sampled by choosing $S\subseteq [d]$ to be a random subset that includes each $i\in S$ independently with probability $\nu$, and picking $y_j\sim\Ucal(\{\pm1\})$ independently for every $j\in S$. The distribution $\Pcal_\theta$ is defined to be the distribution of $X$ such that for $j\in S$, $X_j = y_j$ with probability $1$, and $X_j\sim \Ucal(\{\pm 1\})$ independently for every $j\notin S$.

This problem can be seen as learning a sparse Boolean conjunction since positive samples $x$ satisfy the conjunction $\bigwedge_{j\in S} (x_j==y_j)$.
Our next result characterizes the memorization trade-off for this problem, establishing a faster memorization decay
compared to the previous problems:

\begin{theorem} \label{thm: sparse boolean lb}
In the sparse Boolean hypercube clustering problem $\bP_{\mathrm{sB}}$, assume $\nu=C/\sqrt{d}$ for some sufficiently large absolute constant $C>0$ and  $n\leq c\log d$ for a sufficiently small constant $c>0$. 
Then the following hold:
\begin{itemize}
    \item There exists an algorithm $\Acal$, that given a single sample (i.e. $n=1$) satisfies $\err(\Acal,\bP_{\mathrm{sB}}) \leq 0.01$. 
    \item For any constant $\alpha<1/2$, \[
\mem_n(\bP_{\mathrm{sB}}, \alpha)=\Omega\left(\frac{(1-2\alpha)\log\left(\frac{1-\alpha}{\alpha}\right)}{\nu^2 2^{2n}}\right)
-O(\sqrt{d}\log d) .\]
In particular, $\mem_n(\bP_{\mathrm{sB}})=\Omega\left(\frac{d}{2^{2n}}\right)$.
    \item The lower bound above is nearly-tight: there is a learning algorithm $\Acal$ such that $\err(\Acal,\bP_{\mathrm{sB}}) \leq 0.01$ and $\mem_n(\Acal, \bP_{\mathrm{sB}})=O(d\log(d)/2^{2n})$.
\end{itemize}

\end{theorem}

The proof of Theorem~\ref{thm: sparse boolean lb} appears in Appendix~\ref{sec: sparse bool proof}. To prove the memorization lower bound, we establish a test/train $\rho_n$-SDPI for $\rho_n\approx\nu^2 2^{2n}$.
To do so, we introduce a dominating variable $Z_X^{\train}\sim\BSC_{\half+\xi}(X)$ for $X\sim\Pcal_\theta,~\xi\approx \nu 2^{n}$, and show that processing it into $\tilde{X}_{1:n}$ by fixing each coordinate with some suitable probability, or else drawing each sample independently along that coordinate, results in a training set which is distributed identically to $X_{1:n}$. We note that this is a more direct way to establish the test/train SDPI than the one we give in Theorem~\ref{thm: Z SDPI}. We detail this approach in Lemma~\ref{lem:simple-postprocess-SDPI} and Theorem~\ref{thm: excess mem_direct}.
Then, we directly upper-bound $I(\Acal(X_{1:n}),\theta)$ by the entropy of $\theta$ that satisfies $H(\theta)\lesssim\sqrt{d}\log d$. We then conclude that $\mem_n(\Acal, \bP_{\mathrm{sB}})\geq I(\Acal(X_{1:n});X_{1:n}) - H(\theta) \gtrsim
\frac{1}{\rho_n}-\sqrt{d}\log d=
\Omega\left( \frac{d}{2^{2n}}\right)$,
the latter holding under our assumptions on $\nu$ and $n$.
The nearly-matching upper bound follows by considering an algorithm which only stores a subset of $O(d/2^{2n})$ coordinates in which the sample is constant, and arguing that sufficiently many of them are indeed in $S$ with high probability, which suffices for generalization via standard concentration bounds.

\section{Lower Bounds for Mixtures of Clusters}
\label{sec: multiclass}
We now consider ``mixture-of-clusters'' generalization of our learning setting similar to that defined in \citep{feldman2020does,brown2021memorization}. In this setting, data is sampled from some unknown mixture of clusters. The learner however has a prior over the distribution of frequencies of the clusters. More formally, let $\bP=(\Pcal_\theta)_{\theta\sim \Psi}$ denote a problem in the binary classification setting defined in Section~\ref{sec: formal setting}.  Recall, that in this setting a distribution $\Pcal_\theta$ is drawn from a meta-distribution $\theta\sim \Psi$. The learning algorithm is then given a number of examples $X_{1:n}\sim \Pcal_{\theta}^n$ and needs to classify a fresh example as coming from $\Pcal_{\theta}$ or the ``null'' distribution $\Pcal_{0}$ (defined as the marginal of $X\sim \Pcal_{\theta'}$ where $\theta'\sim \Psi$).

For a natural number $k\in\NN$ representing the number of clusters, we model the prior information about frequencies of clusters using a meta-distribution $\Pi$ over $\Delta([k])$ (i.e., $\Pi$ is a distribution over distributions on $[k]$). We define a multi-cluster version of $\bP$, denoted by $\bPmult = (\{\Pcal_\theta\}_{\theta\sim \Psi}, \Pi, k)$ as follows. First, $k$ random cluster parameters $\theta_{1:k} = (\theta_1,\dots, \theta_k)$ are drawn i.i.d.~from $\Psi$ and a frequency vector $\pi$ is drawn from $\Pi$. Then, let $\Pcal_{\theta_{1:k},\pi}$ be the mixture distribution of $\Pcal_{\theta_i}$'s where $\Pcal_{\theta_i}$ has weight $\pi_i$. Namely, $\Pcal_{\theta_{1:k}, \pi}$ is the distribution of $X\sim \Pcal_{\theta_I}$ where $I\sim \pi$. The learning algorithm is then given training data $X_1,\dots, X_n \sim \Pcal_{\theta_{1:k}, \pi}^n$, and needs to distinguish samples coming from $\Pcal_{\theta_{1:k},\pi}$ from those coming from $\Pcal_0$. Formally, let $(X,Y) \sim \Pcal^{\test}_{\theta, \pi}$ be a distribution such that $Y\sim \{0, 1\}$ is uniformly random. Conditioning on $Y$, we have $X\sim \Pcal_{\theta_{1:k}, \pi}$ if $Y = 1$ and $X\sim \Pcal_{0}$ otherwise. The error of a predictor $h:\Xcal\to \{0,1\}$ is then defined as
\begin{align*}
    \err(h) := \Pr_{(X,Y)\sim \Pcal^\test_{\theta_{1:k},\pi}}[h(X) \ne Y].
\end{align*}
Suppose $\Acal$ is a learning algorithm operating on $n$ samples. We define its (average-case) error as
\begin{align*}
    \err(\Acal,\bPmult) := \E_{\substack{\theta_1,\dots,\theta_k\sim \Psi^k,\pi \sim \Pi \\ X_{1:n}\sim \Pcal_{\theta_{1:k},\pi}^n \\ h\gets \Acal(X_{1:n})}}[\, \err(h) \,].
\end{align*}

As in \citep{feldman2020does,brown2021memorization}, we will only consider product priors $\Pi$ in which frequencies of clusters are chosen independently up to a normalization constant. Specifically, for a distribution $p$ over $[0,1]$, the {\em product prior} $\Pi_p^k$ is defined by independently sampling $p_1,\dots, p_k \sim p^k$, and defining $\pi_i = \frac{p_i}{\sum_{i'} p_{i'}}$.

Note that this setting is slightly different from the setting considered by \cite{brown2021memorization}. There, the algorithm gets \emph{labeled} data $(X,I)$ where $I\sim \pi$ and $X\sim \Pcal_{\theta_I}$. Then, given a test example $X\sim \Pcal_{\theta, \pi}$, the algorithm is tasked to label which cluster was $X$ sampled from (i.e., this is a multi-class classification problem). As we demonstrate below, the two versions of the problems are subject to the same memorization phenomenon up to a factor logarithmic in $k$. Here, we present the result for the binary classification setting (i.e., the algorithm needs to tell whether a point is from any one cluster, or is from $\Pcal_0$), noting that essentially the same proof works for the multi-class clustering setting.

\subsection{Memorization Lower Bound}
Given a learning problem $\bPmult$, we would like to understand the
amount of memorization required to achieve a close-to-optimal error. Let us first consider the natural upper bound of memorization.  To ease our discussion, we assume that the learner gets the \emph{additional} knowledge of the cluster ID of its examples. Namely, the learners gets i.i.d.~examples from the distribution $\widetilde{\Pcal}_{\theta,\pi}$, where an element $(X,i)\sim \widetilde{\Pcal}_{\theta,\pi}$ is sampled by first drawing $i\sim \pi$ and then $X\sim \Pcal_{\theta_i}$.
We note that this assumption only makes our lower bounds stronger since a learning algorithm can always ignore the cluster index information. To solve the multi-cluster problem it suffices to be able to distinguish each of the clusters  $\Pcal_{\theta_1},\dots, \Pcal_{\theta_k}$ from $\Pcal_0$ with low error. Let $\pi \sim \Pi$ be a random frequency vector. For each $i\in [k]$, the number of examples from $\Pcal_{\theta_i}$ is expected to be $n\cdot \pi_i$. Consequently, the amount of memorization is roughly $\mem_{n\pi_i}(\Acal, \bP)$, where $\Acal$ is the algorithm we use on each cluster. By adding up the memorization from different clusters we get $\sum_{i\in [k]} \mem_{n\pi_i}(\Acal, \bP)$ as an upper bound. We show that any nearly-optimal algorithm is subject to a lower bound of essentially the same form.

Before proceeding we will need the following property of product priors from \citep{feldman2020does}.

\begin{lemma}[\citealp{feldman2020does}, Lemma 2.1]\label{lemma: ell coefficient}  
    For a distribution $p$ over $[0,1]$ let $\Pi_p^k$ be the product prior and denote by $\bar{p}$ the marginal distribution of the frequency of (any) element, namely the distribution of $\pi_1$ where $\pi \sim \Pi_p^k$.  Consider the random variable $(\pi, i_1,\dots, i_n)$ where $\pi\sim \Pi_p^k$ and $(i_1,\dots, i_n)\sim \pi^n$. For any sequence of indices $(j_1,\dots, j_n)$ that includes $u\in [k]$ exactly $\ell\in [0,n]$ times, it holds that
    \begin{align*}
    \E_{\pi\sim \Pi_p^k,(i_1,\dots, i_n)\sim \pi^n}[\pi_u \mid (i_1,\dots, i_n) = (j_1,\dots, j_n) ] = \tau_\ell := \frac{\E_{\alpha\sim \bar p}[\alpha^{\ell + 1}(1-\alpha)^{n-\ell}]}{\E_{\alpha\sim {\bar p}}[\alpha^{\ell}(1-\alpha)^{n-\ell}]}.
    \end{align*}
\end{lemma}

To interpret Lemma~\ref{lemma: ell coefficient}, suppose a learner has the prior knowledge of $p$ and sees a sequence of examples $(X_1,i_1),\dots, (X_n, i_n)$ in which the cluster $u$ appears exactly $\ell$ times. Lemma~\ref{lemma: ell coefficient} tells us that the expectation of the posterior value of the frequency of $u$ is $\tau_\ell$. Therefore, it is natural to guess that if the learner does not perform well on the $u$-th cluster, it will translate into an error of $\tau_{\ell}$ for the learner.
We will show next that the previously described intuition is indeed correct.

To simplify the discussion, we focus here on the problem settings that are learnable with \emph{vanishing} error. Namely, we make the following assumption:
\begin{assumption}\label{assump: multi learning}
    Let $\bP = \{\Pcal_\theta\}_{\theta \sim\Psi}$ be a learning problem. We assume that $\bP$ is such that, there exists an algorithm that given a single example $X\sim \Pcal_\theta$,  with probability $1 - \mathrm{neg}(n, k, d)$ over $\theta\sim \Psi$ and $X\sim \Pcal_\theta$,  achieves classification error of $o(\frac{1}{k^2})$.
\end{assumption}

We remark that Assumption~\ref{assump: multi learning} is not too restrictive: for all applications we have investigated in this paper (i.e.~Gaussian clustering, Boolean clustering, and Sparse Boolean clustering), the assumption can be satisfied by increasing the correlation parameter by a factor logarithmic in $k$ (see Remark~\ref{rem: error log}).

\paragraph{From sub-optimality to memorization:}
To count the error for clusters of specific sizes we introduce some notation.
Let $S=((X_1,i_1),\dots, (X_n, i_n))$ be a sequence of examples. Having observed $S$, this induces a posterior distribution on $\theta_{1:k}$ and $\pi_{1:k}$. Furthermore, it is easy to observe that $\theta_{1:k}$ and $\pi_{1:k}$ are \emph{independent}. Let $\theta_{1:k}\mid S$ and $\pi_{1:k}\mid S$ be the posterior distributions. For any $\ell\in [0, n]$, let $I_{n\#\ell}(S)\subseteq[k]$ be the set of clusters $i$ that appear exactly $\ell$ times in the sequence. For any model $h:\Xcal \to \{0,1\}$, we define its expected error on the set of clusters $i$ that appear exactly $\ell$ times in $S$ by
\begin{align*}
\errn(h, \theta_{1:k}, S, \ell) := \frac{1}{2} \left( \sum_{i\in I_{n\#\ell}(S)} \Pr_{X\sim \Pcal_{\theta_i}}[ h(X) = 0 ] + \Pr_{X'\sim\Pcal_{0}}[h(X') = 0] \right),
\end{align*}
and the expectation of this error on the posterior distribution $\theta_{1:k}\mid S$ as 
\begin{align*}
\errn(h, S, \ell) := \E_{\theta_{1:k}\mid S} \left[ \errn(h, \theta_{1:k}, S, \ell) \right].
\end{align*}
Similarly for an algorithm $\Acal$, we define
\begin{align*}
\errn(\Acal, \theta_{1:k}, S, \ell) := \E_{h\sim \Acal(S)}\left[ \errn(h, \theta_{1:k}, S, \ell) \right] .
\end{align*}
and
\begin{align*}
\errn(\Acal, S, \ell) := \E_{\theta_{1:k}\mid S} \left[ \errn(\Acal, \theta_{1:k}, S, \ell)\right] .
\end{align*}

One can observe that these two definitions do \emph{not} depend on the prior $\Pi_p^k$. We also define $\mathrm{opt}(\bPmult\mid S)$ (resp.~$\mathrm{opt}(\bPmult)$) to be the minimum of $\err(\Acal, \bPmult, S)$ (resp.~$\mathrm{err}(\Acal, \bPmult)$).

The following theorem shows that the sub-optimality of any learning algorithm can be expressed in terms of its expected error on posterior distributions. 
\begin{theorem}\label{thm: multi cluster error local to global}
    Let $\bPmult=(\{\Pcal_\theta\}_{\theta\sim \Psi}, \Pi_p^k, k)$ be a learning problem where $\bP$ is subject to Assumption~\ref{assump: multi learning}. For any learning algorithm $\Acal$, with high probability over a data set $S\in (\Xcal \times [k])^n$, it holds that
    \begin{align*}
        \err(\Acal, \bPmult \mid S) \ge \mathrm{opt}(\bPmult \mid S) + \sum_{1\le \ell \le n} \tau_\ell \cdot \errn(\Acal, S, \ell) - O(1/k).
    \end{align*}
    In particular, it follows that
    \begin{align*}
        \err(\Acal, \bPmult) \ge \mathrm{opt}(\bPmult)  + \E_{\theta_{1:k}\sim \Psi^k, \pi\sim \Pi_p^k, S\sim \widetilde{\Pcal}_{\theta, \pi}}\left[ \sum_{1\le \ell\le n} \errn(\Acal,\theta_{1:k},S,\ell) \right] - O(1/k).
    \end{align*}
\end{theorem}

Theorem~\ref{thm: multi cluster error local to global} says that, if an algorithm $\mathcal{A}$ is sufficiently close to being optimal, then $\errn(\Acal, S, \ell)$ must be small on average. We next show that low average $\errn(\Acal, S, \ell)$ implies memorization lower bounds. For this, we will rely on the fact that memorization lower bounds for the problems we consider scale at least linearly with the advantage over random guessing (namely, $1/2-\err$).
\begin{assumption}\label{assump: error memorization scale}
    Let $\bP$ be a (binary) classification problem. We assume that there exists a constant $c_\bP$ such that $\mem_{\ell}(\bP, \alpha) \geq c_\bP \cdot (1-2\alpha) \cdot \mem_\ell(\bP)$ for every $\alpha \in (0, 1/2)$. 
\end{assumption}
This assumption is satisfied (up to the lower order terms) by all the learning problems we have investigated (see Theorem~\ref{thm: excess mem}). We remark that for this assumption to hold in the case of approximate SDPIs we additionally need to constrain the size of the model output by the algorithm.

The following theorem is our main lower bound. It expresses the memorization for the multi-cluster problem $\bPmult$ as the sum of memorization lower bounds for individual clusters. As in the case of lower bounds for a single cluster classification problem $\bP$, the lower bound is scaled by the advantage over random guessing that the algorithm achieves for clusters of each size. Specifically, for clusters of size $\ell$, the average advantage over random guessing of $\Acal$ when given $S$ is equal to $|I_{n\#\ell}(S)|/2 -  \errn(\Acal,\theta_{1:k},S,\ell)$.
\begin{theorem}\label{thm: accuracy to memorization}
    Let $\bPmult=(\{\Pcal_\theta\}_{\theta\sim \Psi}, \Pi_p^k, k)$ be a learning problem where $\bP$ is subject to Assumptions~\ref{assump: multi learning} and~\ref{assump: error memorization scale}. Let $S = (X_1,i_n),\dots, (X_n,i_n)$ be a dataset of $n$ i.i.d.~ examples from $\widetilde{\Pcal}_{\theta_{1:k}, \pi}$ for $\theta_{1:k}\sim \Psi^k, \pi\sim \Pi_p^k$. For every algorithm $\Acal$, $\mem_n(\Acal, \bPmult) = I(\Acal(S) ; S \mid \theta_{1:k}, \pi)$ satisfies
    \begin{align*}
        \mem_n(\Acal, \bPmult)  \geq c_\bP \cdot \E_{\theta_{1:k}\sim \Psi^k, \pi\sim \Pi_p^k, S\sim \widetilde{\Pcal}_{\theta_{1:k}, \pi}}\left[ \sum_{1\le \ell\le n} \left(|I_{n\#\ell}(S)| - 2\cdot \errn(\Acal,\theta_{1:k},S,\ell) \right) \cdot  \mem_{\ell}(\bP) \right] ~.
    \end{align*}
\end{theorem}
\paragraph{Application example.}
Theorem~\ref{thm: accuracy to memorization} generalizes prior works \citep{feldman2020does,brown2021memorization} by allowing us to reason about memorization of larger clusters (instead of only singleton clusters) in the multi-cluster context.
We will now briefly describe a scenario where our new lower bounds offer a significantly better understanding of memorization. Take $\bP$ to be a binary classification task (e.g., Gaussian/Boolean clustering, or the sparse Boolean clustering as considered by \cite{brown2021memorization}). Let $\Pi_p^k$ be a product prior induced by the singleton distribution $p$ that always outputs $1$. We consider the multi-cluster learning task of $\bPmult = (\bP, \Pi_{p}^k, k)$.

The lower bound in \citep{brown2021memorization} is demonstrated for the training sample size $n=O(k)$, in which case we expect to observe many clusters with only a single example (singleton clusters). In order to achieve a close-to-optimal accuracy,\footnote{Note that when $n < O(k)$, with high probability there will be some clusters not present in the training data. Therefore, no algorithm can achieve a vanishing classification error.} a learning algorithm must perform well on the singleton clusters and thus needs to memorize $\Omega(dn)$ bits. Our general technique recovers this lower bound (up to a logarithmic factor needed to ensure that Assumption \ref{assump: multi learning} holds). However, if one just slightly increases $n$ from $k$ to $k\log k$, the probability of observing a singleton cluster quickly approaches zero and thus the results in \citep{brown2021memorization} do not lead to a meaningful memorization lower bound.   

In the latter regime, we will observe $\Omega(k)$ (i.e., most) clusters of size on the order $\log k$. As we have already described in Section~\ref{sec: app}, for several canonical clustering problems (such as Gaussian/Boolean clusters), with $d^{o(1)}$ training examples, the memorization remains significant, roughly $d/\ell$ for each cluster of size $\ell$. Thus in this regime Theorem~\ref{thm: accuracy to memorization} implies that algorithms that achieve (positive) constant advantage over random guessing will need to memorize $\tilde \Omega(kd/\log k) = \tilde \Omega(nd)$ bits.

\section{Discussion} \label{sec: discussion}

All in all, our proof techniques and lower bounds for specific problem instances support the following intuition regarding the phenomenon of data memorization observed in practice: The tail of real-world data distributions contains many subproblems for which relatively little data is available.
Non-trivial accuracy on these subproblems can be achieved by exploiting all the relatively weak correlations present between points in the dataset and unseen points from the same subproblem. This, however, requires memorizing (almost) all the features of the available data points.
More data allows the learning algorithm to average out some of the inherent randomness (or ``noise'') in the features of the given examples, thus increasing the correlations with the features of an unseen point.
In turn,
this allows the learning algorithm to memorize fewer features of the training data for that subpopulation, specifically those with the strongest correlations.

Our work leaves open directions for future work.
First, we believe our general framework can be applied to other problem instances, and the reduction
is aimed at making our results of general use.
In particular, it is interesting to understand whether there are problems that exhibit an even slower memorization decay than those we present here.

\subsection*{Acknowledgments}
We are grateful to Tor Lattimore \citep{Lattimore-gaussian} for permission to include his algorithm for Gaussian clustering in Section~\ref{sec:gaussian-tor} in the paper.
GK is supported through an Azrieli Foundation graduate fellowship. XL is supported by a Google fellowship.


\bibliographystyle{plainnat}
\bibliography{bib}

\appendix

\newpage

\section{Proofs for Section~\ref{sec: SDPI and MEM general}} \label{app: additional proofs}

\subsection{Proof of Proposition~\ref{prop: approximate sdpi}} \label{sec: proofs framework}

Note that the statement for $\delta=0$ is simply the definition of an SDPI, so without loss of generality we can assume that $\delta>0$.
Thus, we can further assume that $|\Mcal|<\infty$, since otherwise the statement trivially holds.
We start the proof by proving two information theoretic results,
which intuitively,
show that if two random variables are statistically close, then (1) they have roughly the same Shannon entropy, and (2) any fixed randomized algorithm/channel extracts roughly the same amount of information from either source.

\begin{lemma}\label{fact:tv-imply-entropy-bound} 
    Suppose $P,Q$ are two random variables
    with pmfs $p,q$ respectively, supported on a finite domain $\Mcal$
    such that $d_{\TV}(p,q)\le \delta$. Then, we have $|H(P)-H(Q)|\le 2\delta \log\left( \frac{|\Mcal|}{\delta} \right)$.
\end{lemma}

\begin{proof}[Proof of Lemma~\ref{fact:tv-imply-entropy-bound}]
For brevity, write $p_y = \Pr[P = y]$ and $q_y = \Pr[Q = y]$. Since the function $x\mapsto x\log(1/x)$ is monotone in $(0, 1)$, we define sets $L = \{y : p_y > q_y\}$ and $R = \{y : q_y > p_y\}$ and deduce that
\begin{align}
|H(P) - H(Q)|
& \le \max\left\{ \sum_{y\in L} p_y \log(1/p_y) - q_y \log(1/q_y),  \sum_{y\in R} q_y \log(1/q_y) - p_y \log(1/p_y) \right\} \notag \\
&\le \max\left\{  \sum_{y\in L} (p_y - q_y) \log(1/p_y),  \sum_{y\in R} (q_y - p_y) \log(1/q_y) \right\}. \label{equ:entropy-diff-to-tv-intermediate}
\end{align}
We show how to upper bound the first term $\sum_{y\in L} (p_y - q_y) \log(1/p_y)$. The bound for the second term can be analogously established. Consider the set $S_P = \{y \in L : p_y < \frac{\delta}{|\Mcal|} \}$. We have
$$
\begin{aligned}
\sum_{y\in L} (p_y - q_y) \log(1/p_y)
&\le \sum_{y\in L\setminus S_P} (p_y - q_y) \log(1/p_y) + \sum_{y\in S_P} p_y \log(1/p_y) \\
&\le \sum_{y\in L\setminus S_P} (p_y - q_y) \log\left(\frac{\delta}{|\Mcal|} \right) + \sum_{y\in S_P} p_y\log(1/p_y) \\ 
&\le \delta \log\left( \frac{|\Mcal|}{\delta} \right) +  \sum_{y\in S_P} p_y\log(1/p_y).
\end{aligned}
$$
By the monotonicity of $x\mapsto x\log(1/x)$, the bound of $p_y<\frac{\delta}{|\Mcal|}$, and the cardinality bound of $|S_P|$, we obtain
$$
\sum_{y\in S_P} p_y\log(1/p_y)\le \sum_{y\in S_P} \frac{\delta}{|\Mcal|} \cdot \log\left( \frac{|\Mcal|}{\delta} \right) \le \delta \log\left( \frac{|\Mcal|}{\delta} \right).
$$
Overall, we conclude that $\sum_{y\in L} (p_y - q_y) \log(1/p_y) \le 2\delta \log\left( \frac{|\Mcal|}{\delta} \right)$. The bound for the second term of \eqref{equ:entropy-diff-to-tv-intermediate} can be similarly established, concluding the proof.
\end{proof}

\begin{lemma} \label{lem: tv I robust}
    Let $B_1,B_2$ be two random variables over some space $\Bcal$ such that $d_{\TV}(B_1,B_2)\leq\delta$, and suppose $\Acal\colon\Bcal\to\Mcal$ is some randomized algorithm. Then
    \[
    \left|I(\Acal(B_1);B_1)-I(\Acal(B_2);B_2)\right|\leq 4\delta\log(|\Mcal|/\delta).
    \]
\end{lemma}

\begin{proof}[Proof of Lemma~\ref{lem: tv I robust}]
It holds that
$$
\left| I(\Acal(B_1);B_1)-I(\Acal(B_2);B_2) \right| \le \left| H(\Acal(B_1)) - H(\Acal(B_2)) \right| +  \left| H(\Acal(B_1)\mid B_1) - H(\Acal(B_2)\mid B_2) \right|.
$$
By Lemma~\ref{fact:tv-imply-entropy-bound} and the observation $d_{\TV}(\Acal(B_1),\Acal(B_2))\le d_{\TV}(B_1,B_2)\le \delta$, we obtain
$$
\left| H(\Acal(B_1)) - H(\Acal(B_2)) \right| \le 2\delta \log\left( \frac{|\Mcal|}{\delta} \right).
$$
We also note that
$$
\left| H(\Acal(B_1)\mid B_1) - H(\Acal(B_2)\mid B_2) \right|
\le \sum_{z\in \mathcal{B}} \left| \Pr[B_1 = z] - \Pr[B_2=z] \right| \cdot H(\Acal(z)) \le 2\delta \log(|\Mcal|).
$$
Combining both inequalities completes the proof of Lemma~\ref{lem: tv I robust}.
\end{proof}

Given the lemmas above, we are ready to complete the proof of Proposition~\ref{prop: approximate sdpi}.
Let $\tilde{B}$ be the approximation of $B$ which is given by the SDPI assumption, and denote $M:=\Acal(B),~\tilde{M}:=\Acal(\tilde{B})$. Applying Lemma~\ref{lem: tv I robust} and the $\rho$-SDPI assumption that holds for $(A,\tilde{B})$, we see that
\[
I(M;B)
\geq I(\tilde{M};\tilde{B})-4\delta\log(|\Mcal|/\delta)
\geq \frac{1}{\rho}I(\tilde{M};A)-4\delta\log(|\Mcal|/\delta)~.
\]
Further note that $d_{\TV}(\tilde{M},M)\leq d_{\TV}(\tilde{B},B)\leq\delta$, which implies $|H(\tilde{M})-H(M)|\leq2\delta\log(|\Mcal|/\delta)$ by Lemma~\ref{fact:tv-imply-entropy-bound}. Furthermore, by assumption, the same is true even after we condition on $A$. Therefore, we may conclude that 
$$|I(\tilde{M};A)-I(M;A)| \leq |H(\tilde M) - H(M)| + |H(M\mid A) - H(\tilde M\mid A)| \leq 4\delta\log(|\Mcal|/\delta).$$
Plugged into the inequality above, we get that
\[
I(M;B)
\geq \frac{1}{\rho}\left(I(M;A)-4\delta\log(|\Mcal|/\delta)\right)-4\delta\log(|\Mcal|/\delta)~,
\]
or rearranged,
\[
I(M;A)\leq
\rho I(M;B)+4\delta(1+\rho)\log(|\Mcal|/\delta)
\leq\rho I(M;B)+8\delta\log(|\Mcal|/\delta)
~.
\]

\subsection{Completing the Proof of Theorem~\ref{thm: excess mem}}\label{sec: complete excess mem proof}

We will show that if $\err(\Acal, \bP) \le \alpha$, then
\begin{align*}
    I(\Acal(X_{1:n}) ; X) \ge 
    \kl{\mathrm{Ber}(1-\alpha)}{\mathrm{Ber}(\alpha)} = 
(1-2\alpha)\log\left(\frac{1-\alpha}{\alpha}\right)~.
\end{align*}

Consider a random variable $X' \sim \Pcal_0$ drawn independently of $X\sim\Pcal_\theta,\,X_{1:n}\sim\Pcal_\theta^n$ and $\Acal(X_{1:n})$. Since $\Pcal_0$ is the marginal of $\Pcal_\theta$ over $\theta$, the characterization of the mutual information as the KL-divergence between the joint and product distributions shows that

\begin{align*}
    I(\Acal(X_{1:n});X) 
    &= \kl{\Acal(X_{1:n}), X}{\Acal(X_{1:n}), X'}
    \\&\ge \kl{\Acal(X_{1:n})(X)}{\Acal(X_{1:n})(X')}
    \\&=\kl{\mathrm{Ber}(p)}{\mathrm{Ber}(q)}~,
\end{align*}
where the inequality follows because post-processing does not increase the KL divergence,
and $p,q$ denote the probability of the model classifying positive/negative points as $1$, respectively. Now, the error of $\Acal$ is equal to $\frac{1-p}{2} + \frac{q}{2}$ and therefore the condition $\err(\Acal, \bP) \le \alpha$ implies that $p-q \leq 1- 2\alpha$. Optimizing this expression (as in the proof of Fano's inequality, cf. \citealp{cover1999elements}) we get \begin{align*}
I(\Acal(X_{1:n});X)
\geq \kl{\mathrm{Ber}(1-\alpha)}{\mathrm{Ber}(\alpha)} = 
(1-2\alpha)\log\left(\frac{1-\alpha}{\alpha}\right)\ .
\end{align*}

\section{Proof of Gaussian Clustering Application}
\label{sec: gaussian proof}

In this section, we prove Theorem~\ref{thm: gaussian lb}. We will establish the three claims one by one.

\subsection{Gaussian sample complexity $(n=1)$}\label{sec: gauss n=1}

Given a single sample $X_1\sim\Pcal_\theta$, we will show that the algorithm that returns the linear classifier
\[
h(X):=\one{\inner{X_1,X}\geq \sqrt{4\log(200)d}}
\]
has expected error at most $0.01$. This will follow from standard concentration bounds which we recall below (cf. \citealp{vershynin2018high}).

\begin{lemma} \label{lem: gaussian bounds}
Suppose $X,Y\overset{iid}{\sim} \Ncal(0_d, I_d)$ are independent isotropic Gaussians. Then the following hold:
\begin{enumerate}
    \item $\Pr[\inner{X,Y}\geq t]\leq \exp(-t^2/2d)$ for all $t>0$;
    \item
    $\Pr[\|X\|^2\leq  \gamma d]\leq \exp(-(1-\gamma)^2 d/4)$ for all $\gamma\in(0,1)$.
\end{enumerate}
\end{lemma}

We start by noting that in the null case, $X_1$ and $X\sim\Pcal_0$ are simply two independent isotropic Gaussians, and therefore Lemma~\ref{lem: gaussian bounds}.1 
ensures that
\[
\Pr_{X\sim\Pcal_0}[\inner{X_1,X}\geq \sqrt{4\log (200)d}]\leq e^{-4\log(200)/4}=\frac{1}{200}~.
\]
On the other hand, if $X\sim\Pcal_\theta$ then both $X$ and $X_1$ can be seen as distributed as $X=\lambda\theta+g_0,~X_1=\lambda\theta+g_1$ where $g_0,g_1\sim\Ncal(0_d,(1-\lambda^2)I_d)$ are independent of one another as well as of $\theta$. Hence,
\begin{align}
\inner{X,X_1}&=\lambda^2\|\theta\|^2+\lambda\inner{\theta,g_0}+\lambda\inner{\theta,g_1}+\inner{g_0,g_1} \nonumber
\\&=\frac{C^2}{d^{1/2}}\|\theta\|^2+\frac{C}{d^{1/4}}\inner{\theta,g_0}+\frac{C}{d^{1/4}}\inner{\theta,g_1}+\inner{g_0,g_1}
~.
\label{eq: gaussian 3 terms}
\end{align}
Since $1-\lambda^2<1$, applying Lemma~\ref{lem: gaussian bounds}.1
to the latter three summands and union bounding ensures that
\[
\Pr_{\theta,g_0,g_1}[\min\{\inner{\theta,g_0},\inner{\theta,g_1},\inner{g_0,g_1}\}<-\sqrt{4\log(1200)d}]
\leq 3e^{-\log(1200)}
=\frac{1}{400}~.
\]
Furthermore,
Lemma~\ref{lem: gaussian bounds}.2
ensures that for $d$ larger than some absolute constant, it holds that
$\Pr[\|\theta\|^2<d/2]<\frac{1}{400}$ . Union bounding over this event as well, we see that with probability at least $1-\frac{1}{200}:$
\[
\inner{X,X_1}\geq \half C^2 d^{1/2}-2C\sqrt{4\log(1200)}d^{1/4}-\sqrt{4\log(1200)d}~.
\]
For sufficiently large absolute constant $C$, the latter is larger than $\sqrt{4\log(200)d}$, 
and we see that the linear classifier $h$ achieves error at most $2\cdot\frac{1}{200}=\frac{1}{100}$, completing the proof.

\subsection{Gaussian memorization lower bound}

In order to prove the lower bound, our goal is to
establish that the conditions in Theorem~\ref{thm: excess mem} hold. To do so, we use the dominating variables approach and apply
Theorem~\ref{thm: Z SDPI}, without any need for approximation (i.e. $\delta_n=0)$. We
construct the dominating random variables $Z^{\train}_\theta,Z_\theta^\test$ as
\begin{align}
Z_\theta^{\train}&\sim\Ncal\left(\frac{\lambda\sqrt{n}}{\sqrt{1+(n-1)\lambda^2}}\cdot\theta,\frac{1-\lambda^2}{1+(n-1)\lambda^2}\cdot I_d\right),
\label{eq: Z^n Gaussian}
\\
Z_\theta^\test&\sim\Ncal(\lambda\theta,(1-\lambda^2)I_d).    \nonumber
\end{align}
To see that these variables satisfy the $\rho_n$-SDPI,
we will use Fact~\ref{fact: classic SDPI gauss}, and therefore need to show the variables are $\sqrt{\rho_n}$-correlated
jointly distributed
unit Gaussians, for suitable $\rho_n$.
To that end, we denote by $Z^{\train},~Z^{\test}$ 
the marginals of $Z_\theta^{\train},~Z_\theta^{\test}$ over  $\theta\sim\Ncal(0_d,I_d)$ respectively. Note that \eqref{eq: Z^n Gaussian} shows that $Z_\theta^{\train},~Z_\theta^{\test}$ are jointly Gaussian
whose marginals satisfy
$Z^{\train}, Z^{\test}\sim N(0_d,I_d)$ simply by summing means and variances.
We further note that each coordinate $i\in[d]$ satisfies
\begin{align*}
\E\left[(Z^{\train}_\theta)_i\cdot(Z^\test_\theta)_i\right]
=\frac{\lambda\sqrt{n}}{\sqrt{1+(n-1)\lambda^2}}\cdot\lambda=:\sqrt{\rho_n}
,
\end{align*}
thus establishing the $\rho_n$-SDPI, as claimed.

We turn to argue about the existence of the desired mappings $\Phi^{\train},\Phi^\test$.
First note that $Z_\theta^{\train}$ can be mapped to a sample from $\Ncal(\lambda\theta,\frac{1-\lambda^2}{n}I_d)$ simply by rescaling
by $\sqrt{\frac{n}{1+(n-1)\lambda^2}}$.
We further use a known result
that states that the sample average $\frac{1}{n}\sum_{i=1}^{n}X_i \sim \Ncal(\lambda\theta,\frac{1-\lambda^2}{n}I_d)$ is a sufficient statistic for the mean of a Gaussian distribution with known covariance matrix (cf. \citealp[Section 2.9]{cover1999elements}), rephrased below: 
\begin{lemma}
\label{lem:gauss-mean-suffices}
For any $\sigma >0$ and $n$,  there is an agorithm $\Phi\colon \R^d \to (\R^d)^n$ such that for any $\mu \in \R^d$ and $\hat{X} \in \R^d$,
the output of the algorithm $(X_1,\dots, X_n) = \Phi(\hat{X})$ satisfies  $\hat{X}=\frac{1}{n}\sum_{i=1}^{n}X_i$. Further, if $\hat{X} \sim \Ncal(\mu,\frac{\sigma^2}{n}I_d)$ then $(X_1,\dots,X_n)$ are distributed as independent samples from $\Ncal(\mu, \sigma^2 I_d)$.
\end{lemma}

A composition of scaling and the result above provides $\Phi^{\train}$.
Moreover, $\Phi^\test$ is simply the identity since $Z^\test_\theta\sim\Pcal_\theta$.
We therefore establish the test/train condition in Theorem~\ref{thm: Z SDPI} with $\delta_n\equiv0$.

To establish the data generation SDPI with $\epsilon_n\equiv 0$ as well, we once again use our construction of the post-processing $Z^{\train}_\theta$
as in Eq. (\ref{eq: Z^n Gaussian}), and claim that $\theta$ and $Z^{\train}_\theta$ satisfy the $\tau_n$-SDPI for 
$\sqrt{\tau_n}:=\frac{\lambda\sqrt{n}}{\sqrt{1+(n-1)\lambda^2}}$.
This follows as we know that they are jointly unit Gaussians, and by noting that they are $\sqrt{\tau_n}$-correlated coordinate-wise by construction of $Z^{\train}_\theta$.

Overall, by applying Theorem~\ref{thm: Z SDPI}, we see that the conditions of Theorem~\ref{thm: excess mem} are met, and we get that for any $\alpha<\half$,
\[
\mathrm{mem}_n(\bP_{\mathrm{G}},\alpha)
\geq \frac{1-\tau_n}{\rho_n} C_{\alpha}
=\frac{1-\frac{\lambda^2 n}{1+(n-1)\lambda^2}}{\frac{\lambda^4 n}{1+(n-1)\lambda^2}} \cdot (1-2\alpha)\log\left(\frac{1-\alpha}{\alpha}\right)
=\frac{1-\lambda^2}{\lambda^4 n} \cdot (1-2\alpha)\log\left(\frac{1-\alpha}{\alpha}\right)
~.
\]

\subsection{Gaussian memorization upper bound}

\label{sec:gaussian-tor}
We now describe a matching upper bound for the task of Gaussian clustering. The first upper bound is based on adding Gaussian noise and matches the lower bound in the entire range of $n$. It was communicated to us by Tor Lattimore. For completeness we also include an algorithm based on a projection from an earlier version of this work. It gives an upper bound that is tight up to a logarithmic factor and applies only when $n\leq c \sqrt{d}$ for some constant $c>0$. 
\subsubsection{Gaussian noise addition}

\paragraph*{Setup and the algorithm.} Let $\lambda = Cd^{-1/4}$ where $C$ is some absolute constant. Recall that the goal is to distinguish $\Pcal_{\theta} = \Ncal(\lambda \theta, (1-\lambda^2) I_d)$ from $\Pcal_0 = \Ncal( 0_d, I_d)$ where $\theta\sim \Ncal( 0_d, I_d)$. 

We describe the algorithm: on input $X_{1:n}$ the algorithm calculates $\hat{X} := \frac{1}{n}\sum_{i\in [n]}X_i$ and $\tilde{\mu} = \hat{X} + \eta$ where $\eta \sim \Ncal(0_d,I_d)$ is an independent Gaussian. Then it returns the linear threshold classifier $\Acal(X_{1:n})(X) = \one{ \langle X, \tilde{\mu}\rangle \ge \tau}$ where $\tau$ will be defined later.

\paragraph*{Mutual information bound.} With the setup above, consider the joint random variable $(\theta, X_{1:n},  \hat{X}, \eta, \mathcal{A}(X_{1:n}))$. The classifier $\mathcal{A}(X_{1:n})$ is as a post-processing of $\hat{X}$  and $\eta$ where we recall that $\eta\sim \mathcal{N}(0_d, I_d)$ is an independent Gaussian. Thus,
\begin{align*}
    I(\Acal(X_{1:n}); X_{1:n} \mid \theta) \le I(\hat{X} + \eta; \hat{X} \mid \theta) = \E_{\theta, u \sim \hat{X}}\left[ \kl{u + \eta}{\hat{X} + \eta}\right]
\end{align*}
Note that conditioned on $\theta$ and $u\sim \hat{X}$, the random variable $u + \eta$ is distributed as $\Ncal(u, I_d)$ while the random variable $\hat{X} + \eta$ (with an independently drawn $\hat{X}$)  is distributed as $\Ncal(\theta, (\frac{1-\lambda^2}{n} + 1) I_d)$.

Recall that the KL divergence between from Gaussian $\Ncal(\mu_1,\Sigma_1)$ to Gaussian $\Ncal(\mu_2, \Sigma_2)$ is given by 
\begin{align*}
    \kl{\Ncal(\mu_1,\Sigma_1)}{\Ncal(\mu_2, \Sigma_2)} = \frac{1}{2} \left( \log{\frac{\det \Sigma_2}{\det \Sigma_1}} - d + \mathrm{Tr}(\Sigma_2^{-1}\Sigma_1) + (\mu_2-\mu_1)^\top \Sigma_2^{-1} (\mu_2 - \mu_1) \right).
\end{align*}
We specialize the formula to our setting, yielding
\begin{align*}
     I(\Acal(X_{1:n}); X_{1:n} \mid \theta) 
     &= \E_{\theta\sim \Ncal(0,I_d), u\sim \Ncal(\lambda\theta, \frac{(1-\lambda^2)}{n} I_d) }\left[ \kl{\Ncal(u, I_d) }{ \Ncal\left(\theta, \left(\frac{1-\lambda^2}{n} + 1\right)I_d\right)} \right] \\
     &= \frac{1}{2} \E\left[ d\log\frac{(1-\lambda^2)/n + 1}{1} - d + d \cdot \frac{1}{(1-\lambda^2)/n + 1} + \frac{\|u - \theta\|^2}{(1 + (1-\lambda^2)/n)} \right] \\
     &= \frac{1}{2} \left( d \log(1 + O(1/n)) - d + d \cdot (1 - O(1/n)) + O(d/n) \right) \\
     &= O\left(\frac{d}{n}\right).
\end{align*}
where the approximation hides absolute constants independent of $d$ and $\lambda$, and the inequality holds true for all sufficiently large $d$.

\paragraph*{Accuracy of the algorithm.} We turn to establish the accuracy guarantee of our algorithm. Suppose $X\sim \Pcal_0 = \Ncal(0_d, I_d)$. We see that $\langle X, \tilde{\mu}\rangle$ is distributed as $\Ncal(0, \| \tilde \mu\|^2)$. Recall that $\tilde{\mu} = \hat{X} + \eta $ is distributed as $\Ncal(0_d, (\lambda^2 + \frac{1-\lambda^2}{n} + 1) I_d)$. Then, by using Lemma~\ref{lem: gaussian bounds}, we know that with probability $1-\left( \frac{1}{nd} \right)^{10}$, it holds that $\langle X, \tilde{\mu} \rangle \le O(C\sqrt{d})$ where $C$ is the (large enough) constant from our choice of $\lambda = C d^{-1/4}$.

On the other hand, if $X\sim\Pcal_\theta$ then both $X$ and $\tilde{\mu}$ can be seen as distributed as $X=\lambda\theta+g_0,~\tilde{\mu}=\lambda\theta+g_1$, where $g_0\sim\Ncal(0_d,(1-\lambda^2)I_d),\,g_1\sim\Ncal(0_d,(\frac{1-\lambda^2}{n}+1)I_d)$ are independent of one another as well as of $\theta$. Hence,
\begin{align*}
\inner{X,\tilde{\mu}}&=\lambda^2\|\theta\|^2+\lambda\inner{\theta,g_0}+\lambda\inner{\theta,g_1}+\inner{g_0,g_1}
\\&=\frac{C^2}{d^{1/2}}\|\theta\|^2+\frac{C}{d^{1/4}}\inner{\theta,g_0}+\frac{C}{d^{1/4}}\inner{\theta,g_1}+\inner{g_0,g_1}
~.
\end{align*}
The rest of the proof is therefore implied by the concentration given by Lemma~\ref{lem: gaussian bounds},
similarly to the arguments following \eqref{eq: gaussian 3 terms}, which we repeat here briefly.
Since we have $\|\theta\|_2^2 \ge \frac{d}{2}$ with probability at least $1-\frac{1}{400}$, whereas the three latter inner product terms have magnitude $O(\sqrt{d})$ with the same high probability, it follows that 
$\langle X, \tilde{\mu}\rangle = \Omega(C^2\sqrt{d})$.
For sufficiently large absolute constant $C$, this is clearly larger than $\langle X, \tilde{\mu} \rangle  = O(C\sqrt{d})$ when $X\sim\Pcal_0$.
Hence, by setting the threshold $\tau$ in between (e.g., $\tau=C^{3/2}\sqrt{d}$) we get a linear classifier with small error.

\subsubsection{Projection-based algorithms}

Given a sample $X_{1:n}\sim\Pcal_\theta^n$, we will show that there exists a projection-based algorithm that returns a good classifier with the claimed memorization upper bound, whenever $n=O(\sqrt{d})$.
For that, we argue that an algorithm can return a high accuracy classifier $h=\Acal(X_{1:n})$ which is describable using $O(d \log(\tfrac{d}{n})/n)$ bits. In particular,
\begin{equation} \label{eq: gauss upper mem proof}
\mem_n(\Acal,\bP_{\mathrm{G}})
=I(h;X_{1:n}\mid \theta)
\overset{(\star)}{\leq} I(h;X_{1:n})
\leq H(h)=O\left(\frac{d \log(\tfrac{d}{n})}{n}\right)
~,
\end{equation}
where $(\star)$ is due to the fact that $h\indep \theta\,|\,X_{1:n}$.

To construct the predictor we introduce some notation. Given a vector $v\in\reals^d$, we denote by $v^{[1:\ell]}:=(v_1,\dots,v_\ell,0,\dots,0)\in\reals^d$ its projection onto the first $\ell$ coordinates embedded in $\reals^d$.
Let $\hat{X}:=\frac{1}{n}\sum_{i=1}^{n}X_i$ be the dataset's empirical average.
For $\hat{X}^{[1:\ell]}=(\hat{X}_1,\dots,\hat{X}_\ell,0,\dots,0)\in\reals^d$, let $q_k(\hat{X}^{[1:\ell]})\in\reals^d$ be a quantization of $\hat{X}^{[1:\ell]}$ that stores only $k$ bits of $\hat{X}^{[1:\ell]}$ in each coordinate's binary expansion, allowing up to magnitude of $O(\log (d/n))$. We consider the algorithm $\Acal$ that returns the
linear classifier
\[
h(X):=\one{\inner{X^{[1:\ell]},q_k(\hat{X}^{[1:\ell]})}\geq t}.
\]
We claim that for some
\[
\ell=\frac{d}{n}~,
~~~t
=\Theta(\lambda^2\ell)=\Theta(\sqrt{d}/n)~,
~~~k=\Theta\left(\log(\tfrac{d}{n})\right)~,
\]
this classifier satisfies the desired properties.
We note that this classifier is described by $\ell k=O(d \log(\tfrac{d}{n})/n)$ bits, so
\eqref{eq: gauss upper mem proof} indeed holds, proving the memorization upper bound.

The argument that $h$ is a high accuracy classifier follows the same analysis as in Section~\ref{sec: gauss n=1}, based on the Gaussian concentration given by Lemma~\ref{lem: gaussian bounds}. Indeed, noting that $\hat{X}^{[1:\ell]}\sim \Ncal(\lambda\theta,\frac{1-\lambda^2}{n}I_\ell)$,
then simply repeating the analysis therein ensures that for $X\sim\Pcal_\theta:~\langle X^{[1:\ell]},\hat{X}^{[1:\ell]}\rangle=\Omega(\lambda^2\E\|\theta^{[1:\ell]}\|)=\Omega(\lambda^2\ell)>t$, while in the null case, the inner product has mean zero and deviates by at most $O(\ell(\lambda^2+1/n))<t$ with high probability.\footnote{Note that this is where the requirement that $n=O(\sqrt{d})$ is used.}
Furthermore, the quantization induces negligible error due to that fact that
\begin{align*}
|\inner{X^{[1:\ell]},q_k(\hat{X}^{[1:\ell]})}-\inner{X^{[1:\ell]},\hat{X}^{[1:\ell]}}|
\leq |\inner{X^{[1:\ell]},q_k(\hat{X}^{[1:\ell]})-\hat{X}^{[1:\ell]}}|
\leq \|X^{[1:\ell]}\|\cdot \|q_k(\hat{X}^{[1:\ell]})-\hat{X}^{[1:\ell]}\|~,
\end{align*}
with the latter quantization error term going to zero as $\exp(-\Omega(k))$.

\section{Proof of Boolean Clustering Application}\label{sec: boolean proof}
In this section, we prove Theorem~\ref{thm: boolean lb}
assuming $C\geq(8\log(200))^{1/4}$.

\subsection{Boolean sample complexity $(n=1)$}\label{sec: boolean sample complexity}

Consider the algorithm that given $X_1\sim\Pcal_\theta$ returns the predictor
\[
h(X):=\one{\inner{X_1,X}\geq \sqrt{2\log(200)d}}~.
\]
To see why this predictor suffers from error at most $0.01$, first note that for $X\sim \Pcal_{0}:~\E[\inner{X_1,X}]=0$, and by Hoeffding's inequality:
\begin{equation}\label{eq: boolean null}
\Pr[\inner{X_1,X}\geq\sqrt{2\log(200)d}]\leq \frac{1}{200}~. 
\end{equation}
On the other hand,
for $X\sim P_\theta:$
\begin{align} \label{eq: bool E lower bound}
\E[\inner{X_1,X}]
=\sum_{i=1}^{d}\E[(X_1)_i\cdot (X)_i ]
=d\cdot\E[(X_1)_1]\cdot \E[(X)_1] = d\lambda^2=C^2\sqrt{d} \ ,  
\end{align}
and thus by Hoeffding's inequality
\begin{equation}
\label{eq: boolean true}
\Pr[\inner{X_1,X}< \sqrt{2\log(200)d}]
\leq
\Pr[\inner{X_1,X}< C^2\sqrt{d}-\sqrt{2\log (200) d}]\leq \frac{1}{200}~.
\end{equation}
where we used that fact that $C$ satisfies $\sqrt{2\log(200)}<C^2-\sqrt{2\log(200)}$. By union bounding over \eqref{eq: boolean null} and \eqref{eq: boolean true}
we get that the defined predictor $h$ indeed classifies between samples from $\Pcal_\theta$ and $\Pcal_0$ with error at most $\frac{1}{200}+\frac{1}{200}=\frac{1}{100}$, as claimed.

\subsection{Boolean memorization lower bound}

Once again, we will follow the framework of Theorem~\ref{thm: excess mem} to prove Item 2 of Theorem~\ref{thm: boolean lb}.

\subsubsection*{Binomials and advanced composition}

To start, we state and prove the following lemma.

\begin{lemma}\label{lemma:composition-RR}
    Let $\lambda < 1$
    and $n\in\NN$ be such that $n < \frac{1}{10\lambda^2}$. 
    Let $X$ be distributed according to a product distribution over $\{\pm 1\}^n$ where for each coordinate $i$ it holds that $\Pr[X_i = 1] = \frac{1+\lambda}{2}$. Similarly let $Y$ be a product distribution over $\{\pm 1\}^{n}$ where for each $i\in [n]$ it holds that $\Pr[Y_i = 1] = \frac{1-\lambda}{2}$. Then, for every $\delta \in (0, 1)$ and $\rho = \sqrt{8 n \log(1/\delta)}\lambda$, there exists a pair of distributions $\mathbf{A}, \mathbf{B}$ over $\{\pm 1\}^{n}$ such that
    $$
    \begin{aligned}
    & d_{\TV}(X, \frac{1+\rho}{2}\mathbf{A} + \frac{1-\rho}{2}\mathbf{B}) < \delta, \\
    & d_{\TV}(Y, \frac{1-\rho}{2}\mathbf{A} + \frac{1+\rho}{2} \mathbf{B}) <\delta.
    \end{aligned}
    $$
    Here, $p\mathbf{A} + (1-p)\mathbf{B}$ denotes the mixture of two distributions with weights $p$ and $(1-p)$.
\end{lemma}

Lemma~\ref{lemma:composition-RR} tells us how to post-process a Bernoulli random variable with bias $\frac{1}{2}\pm \rho$, to approximate a sequence of $n$ independent Bernoulli random variables, each of bias $\frac{1}{2}\pm \lambda$ where $\lambda \approx \frac{\rho}{\sqrt{n}}$. 
This lemma appears to be folklore in the differential privacy literature, yet for completeness, we include a brief proof below. First, we need the notion of ``indistinguishability'' between distributions.

\begin{definition}
Let $X, Y$ be a pair of random variables with the same support $\Omega$. We say that $X$ and $Y$ are $(\varepsilon, \delta)$-indistinguishable, if for every measurable $E\subseteq \Omega$, it holds that
$$
\begin{aligned}
& \Pr[X\in E] \le e^{\eps} \Pr[Y\in E] + \delta, \\
& \Pr[Y\in E] \le e^{\eps} \Pr[X\in E] + \delta.
\end{aligned}
$$
\end{definition}

We use the following ``decomposition'' lemma for indistinguishable distributions, the proof of which is usually attributed to \cite{KairouzOV15-Comp}.

\begin{proposition}[\cite{KairouzOV15-Comp}]\label{prop:privacy-decomp}
    Suppose $X,Y$ are $(\eps,\delta)$-indistinguishable. Then, there exists four distributions $X',Y',E_1,E_2$ such that
    $$
    \begin{aligned}
    X = (1-\delta)X' + \delta E_1, \\
    Y = (1-\delta)Y' + \delta E_2,
    \end{aligned}
    $$
    and $X',Y'$ are $(\eps,0)$-indistinguishable. Furthermore, there are two distributions $U,V$ such that
    $$
    \begin{aligned}
    X' = \frac{e^\eps}{1+e^\eps} U + \frac{1}{1+e^\eps} V, \\
    Y' =\frac{e^\eps}{1 + e^\eps} V + \frac{1}{1+e^\eps} U.
    \end{aligned}
    $$
\end{proposition}

In light of Proposition~\ref{prop:privacy-decomp}, to prove Lemma~\ref{lemma:composition-RR}, it suffices to show that the distributions $X,Y$ in the statement are $(\rho,\delta)$-indistinguishable, where $\rho, \delta$ are as in Lemma~\ref{lemma:composition-RR}.

In the language of differential privacy, $X$ and $Y$ can be understood as the $n$-wise parallel composition of the randomized response mechanism,\footnote{Recall that the $(\varepsilon,0)$-DP randomized response (RR) mechanism receives an input $b\in \{\pm 1\}$, and outputs a random bit $\hat{b}$ such that $\Pr[\hat{b} = b] = \frac{e^\eps}{1 + e^\eps}$.} where each individual RR mechanism enjoys $(2\lambda, 0)$-DP.
We can then apply the advanced composition theorem of differential privacy \citep{DworkRV10-boosting,KairouzOV15-Comp}. We present a special case of the theorem below, which is stated in our language and suffices for our purpose.

\begin{proposition}[\cite{DworkRV10-boosting,KairouzOV15-Comp}]\label{prop:advanced-composition}
Suppose $P, Q$ are two distributions that are $(\varepsilon, 0)$-indistinguishable. Let $P^{\otimes n}, Q^{\otimes n}$ be the $n$-fold tensorization of $P$ and $Q$. (Namely, to sample from $P^{\otimes n}$, one independently draws $n$ samples from $P$ and list them in order). Then, $P^{\otimes n}$ and $Q^{\otimes n}$ are $(\sqrt{4n\log(1/\delta)}\eps, \delta)$-indistinguishable.
\end{proposition}

Finally, combining Propositions~\ref{prop:advanced-composition} and~\ref{prop:privacy-decomp} concludes the proof of Lemma~\ref{lemma:composition-RR}.

\subsubsection*{The mutual information bound}

We are ready to establish the second item of Theorem~\ref{thm: boolean lb}.
Recall the setting of $\LB$: a random $\theta\sim \Ucal(\{\pm 1\}^{d})$ is drawn. The input distribution $\Pcal_\theta$ is the distribution $\BSC_{\frac{1-\lambda}{2}}(\theta)$. Namely, to draw $y\sim \Pcal_\theta$, for each coordinate $i\in [d]$ set $y_i = \theta_i$ with probability $\frac{1+\lambda}{2}$, and set $y_i = -\theta_i$, otherwise. Next, we verify that the two conditions of Theorem~\ref{thm: excess mem} are met for $\LB$.

\medskip\noindent\textbf{Data generation SDPI.}
Conditioning on $\theta\in \{\pm 1\}^{d}$, for each $i\in [d]$, the $i$-th coordinates of $X_{1: n}$ are $n$ independent draws from a Bernoulli distribution of bias $\frac{1}{2} + \lambda \cdot \theta_i$. As such, we apply Lemma~\ref{lemma:composition-RR} with some $\delta > 0$ to be specified. Namely, we define $Z_\theta^{\train} = Z$ by drawing each single $Z_i$ from $\BSC_{\frac{1-\xi}{2}}(\theta_i)$
with $\xi = \sqrt{8 n \log(1/\delta)} \cdot \lambda$, and post-process $Z_i$ as in Lemma~\ref{lemma:composition-RR}, to obtain a collection of $n$ Boolean variables $(\tilde{X}_{1,i},\dots, \tilde{X}_{n,i})$, such that (over the randomness of $Z_i$ and the post-processing) the TV distance between $(X_{1,i},\dots, X_{n,i})$ and $(\tilde{X}_{1,i},\dots, \tilde{X}_{n,i})$ is at most $\delta$. We run the same argument for each coordinate $i\in [d]$. Overall, we have shown how to post-process a random variable $Z_\theta^{\train}\sim \BSC_{\frac{1-\xi}{2}}(\theta)$ to obtain a sequence of $n$ samples $\tilde{X}_{1:n}$ such that the TV distance between $\tilde{X}_{1:n}$ and $X_{1:n}$ is bounded by $\delta d$.
Using Item 1 of Theorem~\ref{thm: Z SDPI}, this implies that the pair $(\theta, X_{1:n})$ satisfies $(\xi^2, \delta d)$-SDPI 
where $\xi = \sqrt{ 8 n\log(d/\delta)} \cdot \lambda$ and $\delta \in (0, 1)$ can be arbitrarily chosen.

\medskip\noindent\textbf{Test/train SDPI.} We now verify the second condition. Consider the random variable $Z$ we have introduced. We have shown a post-processing mapping $\Phi^{\train}$ such that $\Phi^{\train}(Z)$ is close to $X_{1:n}$ in TV distance. We can also draw $X\sim \Pcal_\theta$ and define $\Phi^{\test}$ as the identity mapping. It remains to show that $Z$ and $X$ satisfy $\rho_n$-SDPI for a small $\rho_n$. Indeed, note that $\theta$ is uniform over $\{\pm 1\}^{d}$. Conditioned on $\theta$, we have $Z\sim \BSC_{\frac{1-\xi}{2}}(\theta)$ and $X\sim \BSC_{\frac{1-\lambda}{2}}(\theta)$, and they are independent conditioned on $\theta$. Therefore, by the symmetry of the definition of $\BSC$, it follows that $Z\sim \BSC_{\frac{1-\xi}{2}}(X)$, which implies that $Z$ and $X$ satisfy $\rho_n$-SDPI with $\sqrt{\rho_n} = \lambda\xi = O(\lambda^2 \sqrt{n\log(d/\delta)}) = O(\sqrt{n\log(d/\delta)/d})$.
Consequently, by item 2 of Theorem~\ref{thm: Z SDPI}, we conclude that the pair $(X, X_{1:n})$ satisfies $(\lambda^2 \xi^2, \delta d)$-SDPI,
where we recall that $\xi^2 \lambda^2 =O(\lambda^4 n \log(d/\delta))=O(n\log(d/\delta)/d)$.

\medskip\noindent\textbf{Conclusion.} We have established both conditions of Theorem~\ref{thm: excess mem}. Using Theorem~\ref{thm: excess mem}, this yields that for any $\delta\in(0,1):$
\begin{align*}
\mathrm{mem}_n(\Acal,\bP_{\mathrm{B}})
\geq 
\frac{1-8\lambda^2n\log(d/\delta)}{8\lambda^4n\log(d/\delta)}\left(C_{\alpha}-8\delta d\right)-8\delta d\log(|\Mcal|/\delta d)
~.
\end{align*}
By setting $\delta=\min\{\half,\,\frac{C_\alpha}{16d},\,(\frac{C_\alpha(1-8\lambda^2n\log(d))}{256d\lambda^4n\log(d)\log(|\Mcal|)})^2\}$ so that $C_{\alpha}-8\delta d\geq C_\alpha/2$ and $8\delta d\log(|\Mcal|/\delta d)\leq\frac{C_{\alpha}(1-8\lambda^2n\log(d/\delta))}{32\lambda^4n\log(d/\delta)}$, the bound above simplifies to
\[
\mathrm{mem}_n(\Acal,\bP_{\mathrm{B}})
\geq \frac{1-8\lambda^2n\log(d/\delta)}{32\lambda^4n\log(d/\delta)}C_{\alpha}
=\Omega\left(\frac{1-8\lambda^2n\log(d\lambda n\log|\Mcal|/C_\alpha)}{\lambda^4n\log(d\lambda n\log|\Mcal|/C_\alpha)}C_{\alpha}\right)~.
\]
In particular,
as long as $8\lambda^2n\log(d\lambda n\log|\Mcal|/C_\alpha)\leq\half$, which holds for $n\leq \frac{1}{16\lambda^2\log(d\log(|\Mcal|)/16\lambda C_\alpha)}$,
and if
$\log|\Mcal|\in\mathrm{poly}(d)$, or equivalently $|\Mcal|\leq \exp(d^{\tilde{C}})$ for any absolute constant $\tilde{C}>0$ fixed in advance,
then by recalling that $\lambda=C/d^{1/4}$ we get
\[
\mathrm{mem}_n(\bP_{\mathrm{B}},\alpha)=
\Omega\left(\frac{1-8\lambda^2n\log(dn/C_\alpha)}{\lambda^4n\log(dn/C_\alpha)}C_{\alpha}\right)
.
\]
For $\alpha = 1/3$ we can further simplify $\mathrm{mem}_n(\bP_{\mathrm{B}})=\widetilde{\Omega}\left(\frac{d}{n}\right)$.

\subsection{Boolean memorization upper bounds}

Here, we describe algorithms to complement the lower bound and establish Item 3 of Theorem~\ref{thm: boolean lb}.
We split the proof into the two considered regimes $n\leq\sqrt{d}$ and $n\gtrsim\sqrt{d}\log d$.

\subsubsection*{Algorithm when $n\leq\sqrt{d}$}
We will show that when $1\le n\le \sqrt{d}$ examples are available, we can design an algorithm $\Acal$ which learns an accurate model $h$ such that $I(h; X_{1:n}) = O\left( \frac{d}{n} \right)$.

Let $C_{sub} > 0$ be a sufficiently large constant. On input $n$ i.i.d.~samples $X_{1:n}\sim \Pcal_\theta^n$, let $\hat{\theta}$ be the bit-wise majority vote of $X_{1:n}$. Our algorithm chooses $t = \min\left\{ d, \frac{C_{sub}d}{n} \right\}$ and returns the hypothesis
\begin{align}
    h(X) := \one{ \inner{{\hat{\theta}}^{[1:t]},{X^{[1:t]}}} \ge \sqrt{2 \log(200) t} }. \notag
\end{align}
Here, recall that for a vector $v\in \reals^d$, the notation $v^{[1:t]} = (v_1,\dots, v_t, 0, \dots, 0)$ denotes its projection onto the first $t$ coordinates embedded in $\reals^d$. 

We prove that $h$ is, with high probability, a good predictor. First, for $X\sim \Pcal_{0}$ it holds that $\Pr_{X\sim \Pcal_0}[h(X) = 1] \le \frac{1}{200}$ by Hoeffding's inequality (as in Section~\ref{sec: boolean sample complexity}).

It remains to show that $h$ classifies most $X\sim \Pcal_{\theta}$ correctly. This will follow from Slud's inequality \citep{slud1977distribution} paired with a standard Gaussian tail bound, as stated below.

\begin{lemma}
For any coordinate $j\in[d]$, the majority vote satisfies $\Pr[\hat{\theta}_j = \theta_j]\geq \half+\sqrt{1-\exp(-2n\lambda^2)}$.
\end{lemma}

In particular, a Taylor approximation ensures that $\Pr[\hat{\theta}_j = \theta_j]\geq \half+C'\sqrt{n}\lambda$ for some constant $C'$.
Also, $ \Pr[X_j = \theta_j] = \half + \lambda$.
These two combined would imply that
\begin{align}
    \E\left[\langle \hat{\theta}^{[1:t]}, X^{[1:t]} \rangle\right] \ge C'\lambda^2 \sqrt{n} t
    \geq C^2 \sqrt{t}~,
    \notag
\end{align}
the latter holding for sufficiently large absolute constant $C_{sub}$, under the assumption that $n\leq \sqrt{d}$.
Note that the bound above is the same as \eqref{eq: bool E lower bound}, and the proof follows similarly by applying Hoeffding's inequality as in \eqref{eq: boolean true}

To upper bound the excess memorization of $\Acal$, observe that $h$ can be described using $O(\frac{C_{sub}d}{n})$ bits. Hence, we have
\begin{align}
    \mem_{n}(\Acal, \LB) \le I(h; X_{1:n}) \le H(h) = O\left( \frac{d}{n} \right), \notag
\end{align}
as desired.

\subsubsection*{Algorithm for $n\gtrsim\sqrt{d}\log d$}

We now provide an accurate algorithm for which $\mem_n(\Acal,\LB)=O(d^2 n\exp(-nC^2/(2\sqrt{d}))$. In particular, for which there exists a fixed constant $\tilde C$ such that for $n\geq \tilde{C} C^2 \sqrt{d}\log d$, $\Acal$ learns an accurate model $h$ with nearly $0$ excess memorization. 
Similarly to the previously described algorithm, given $n$ i.i.d.~samples $X_{1:n}\sim \Pcal_\theta^n$, let $\hat{\theta}$ be the bit-wise majority vote of $X_{1:n}$. Our algorithm returns the hypothesis 
\begin{align*}
    h_{\hat{\theta}}(X) := \one{ \inner{{\hat{\theta}},{X}} \ge \sqrt{2 \log(200) d} }.
\end{align*}
Noting that this is the same predictor as before but with no projection, it is easy to verify that the accuracy analysis follows from a similar calculation as in the previous case. We now establish the memorization bound.
By definition,
\begin{align*}
    \mem_n(\Acal,\LB)=I(h_{\hat{\theta}};S\mid\theta) 
    \leq H(h_{\hat{\theta}}\mid\theta) =  H(\hat{\theta}\mid\theta) ~.
\end{align*}
We will show that for any fixed setting $\theta'\in \{\pm 1\}^d$, it holds that $$H(h_{\hat{\theta}}\mid\theta = \theta') = O(d^2 n\exp(-nC^2/(2\sqrt{d}))~,$$ thus implying the claimed bound. 
Assume $\theta=\theta'$, and let $\gamma=\Pr_{X_{1:n}}[\hat{\theta}\neq \theta']$ which is the probability that the bit-wise majority vote did not reconstruct $\theta$. 
By Hoeffding's bound, for any coordinate $j\in[d]:\,\Pr_{X_{1:n}}[\hat{\theta}_j\neq \theta'_j]\leq \exp(-n\lambda^2/2)$ so by union bounding we see that $\gamma\leq d\exp(-n\lambda^2/2)$.

This means that the distribution of $\hat\theta$ conditioned on $\theta = \theta'$ has TV distance of at most $\gamma$ to the distribution where $\theta$ is fixed to $\theta'$ (which has $0$ entropy).  Applying Lemma~\ref{fact:tv-imply-entropy-bound},
\begin{align*}
H(\hat{\theta} \mid\theta ) &= \E_{\theta'\sim \Ucal(\{\pm 1\}^d)}[H(\hat{\theta} \mid\theta = \theta' )]  \\
&\leq H(\theta \mid\theta = \theta') + 2\gamma \log(2^d/\gamma) \\
&= 2 \log 2 \cdot \gamma d  + 2\gamma \log(1/\gamma)
\\&= \left(2 \log 2 \cdot d^2 + d n \lambda^2 \right)\exp(-n\lambda^2/2) \\
&= O(d^{2}  n \cdot \exp(-nC^2/(2\sqrt{d}))
~,
\end{align*}
where we used $\lambda=C/d^{1/4}$.

\section{Proof of Sparse Boolean Hypercube Application}\label{sec: sparse bool proof}

In this section, we prove Theorem~\ref{thm: sparse boolean lb}.

\subsection{Sparse Boolean sample complexity $(n=1)$}

The following sample complexity analysis is similar to that of \citet{brown2021memorization}.

Recall that in $\LsB$, the learning problem is parameterized by $\theta = (S,y)$ where $S$ is a random subset including each $i\in [n]$ with probability $\nu = \frac{C}{\sqrt{d}}$ and $y\sim \Ucal(\{ \pm 1\}^{d})$ is chosen uniformly at random. The distribution $\Pcal_{\theta}$ is a product distribution $X$ on $\{\pm 1\}^{d}$ where $X_S=y_S$ with probability one and $X_{\overline{S}}$ is chosen uniformly at random. 

We claim that learning with a single sample is possible as long as the set $S$ satisfies that $|S|\ge \sqrt{4\log(200)d}$. First, note that this event happens with probability at least $1-\exp(-\Omega(\sqrt{d}))$ if we choose $\nu = \frac{C}{d^{1/2}}$ for a large enough constant $C$. Now, suppose we have $|S|\ge \sqrt{4\log(200)d}$ and let $X_{1}$ be the training sample. Consider the hypothesis
\[
h(X) := \one{\langle X, X_1\rangle \ge \sqrt{2\log(200)d}}~.
\]

We verify that $h$ has a low generalization error. 
To see this, we first observe that for $X\sim \Pcal_{0}$, it holds that $\E[\langle X_1,X\rangle] = 0$, and by Hoeffding's bound, we further have $\Pr[\langle X_1,X\rangle\ge \sqrt{2\log(200)d}] \le 0.005$. On the other hand, conditioned on $\theta$, for every $X\sim \Pcal_{\theta}$, we have that $X_S = (X_1)_S$ and $X_{j}\sim \Ucal(\{\pm 1\})$ for every $j\not\in S$. In this case, it is easy to see that $\langle X,X_1\rangle$ is a random variable with mean $|S| \ge \sqrt{4\log(200)d}$
given by a sum of independent $\pm1$'s.
By Hoeffding bound, we get that $\Pr_{X\sim \Pcal_{\theta}}[\langle X,X_1\rangle \le \sqrt{2 \log(200)d}]\le 0.005$. Overall, we see that the predictor $h$ has the expected error of at most $0.01$.
\subsection{Sparse Boolean memorization lower bound}

We now establish item 2 of Theorem~\ref{thm: sparse boolean lb}. Our proof will deviate from the framework of Theorems~\ref{thm: excess mem} and \ref{thm: Z SDPI}. Instead, we will establish a test/train SDPI via a more direct argument that still however relies on our notion of dominating random variables. In addition, we will use a direct upper bound on the description length of $\theta$ to upper bound  $I(X_{1:n};\theta)$ instead of relying on the data generation SDPI. 

We will rely on the following simple version of our reduction to dominating variables in Theorem~\ref{thm: Z SDPI}.
\begin{lemma} \label{lem:simple-postprocess-SDPI}
	Let $\bP=(\Pcal_\theta)_{\theta\sim\Psi}$ be a learning problem,
	and suppose that there exists a random variable $Z_X^{\train}$ (that depends on the value of $X$) such that
    there is a mapping $\Phi^{\train}$ satisfying: the distribution of $(X,\Phi^\train(Z_X^{\train}))$ when $\theta\sim\Psi$ and $X \sim P_\theta$ is identical to the distribution of $(X,X_{1:n})$ when $\theta\sim\Psi$ and $X_{1:n} \sim P_\theta^n$.
	Then if the pair of random variables $(X,Z_X^{\train})$ satisfies $\rho_n$-SDPI, then $(X,X_{1:n})$ for $X\sim\Pcal_\theta$,  $X_{1:n}\sim\Pcal_\theta^n$ and $\theta \sim \Psi$ satisfies $\rho_n$-SDPI.
\end{lemma}
The proof of this lemma follows from the fact that if a pair of random variables $(A,B)$ satisfies $\rho$-SDPI, then for any (randomized) postprocessing function $\Phi$, the pair of  random variables $(A,\Phi(B))$ also satisfies $\rho$-SDPI.

For completeness we also state a variant of Theorem~\ref{thm: excess mem} that bypasses the data generation SDPI (the proof is implicit in the proof of Thm.~\ref{thm: excess mem}). 
\begin{theorem} \label{thm: excess mem_direct}
	Let $\bP=(\Pcal_\theta)_{\theta\sim\Psi}$ be a learning problem satisfying the following test/train SDPI:
the variables $(X,X_{1:n})$ for $\theta \sim\Psi$, $X\sim\Pcal_\theta$ and $X_{1:n} \sim\Pcal_\theta^n$ satisfy $(\rho_n,\delta_n)$-SDPI.
	Then any algorithm $\Acal:\Xcal^n\to\Mcal$ for $\bP$ satisfies the excess memorization bound: 
	\[
	\mathrm{mem}_n(\Acal,\bP)
	\geq \frac{I(\Acal(X_{1:n});X) - 8\delta_n\log(|\Mcal|/\delta_n)}{\rho_n} - I(\Acal(X_{1:n});\theta) \ .
	\]
	Moreover, for any $\alpha<\half:$
	\[
	\mem_n(\bP,\alpha)
	\geq \frac{C_{\alpha}-8\delta_n\log(|\Mcal|/\delta_n)}{\rho_n}  - I(\Acal(X_{1:n});\theta)
	~,
	\text{~~~where~~~}C_{\alpha} := (1-2\alpha)\log\left(\frac{1-\alpha}{\alpha}\right)~.
	\]
\end{theorem}

\paragraph*{The test/train SDPI via post-processing.} Suppose $\theta\sim \Psi$ and $( X,X_{1:n})\sim \Pcal_{\theta}^{n+1}$. We show that the pair $(X,X_{1:n})$ satisfies $\rho_n$-SDPI with $\rho_n = \Theta(2^{2n}/d)$.

We begin with a basic observation: recall that the parameter $\theta$ consists of a pair $(S, y_S)$ where each element is included in $S$ with probability $\nu = \Theta\left(\sqrt{d}\right)$.
Here, conditioning on $X$ would \emph{not} change the distribution of $S$: namely, it is still the case that each $i$ is in $S$ with probability $\nu$. Furthermore, if $i\in S$ then we know the $i$-th bits of the training examples (namely, $X_{1,i},X_{2,i},\dots, X_{n, i}$) all agree with $X_{i}$ 
Otherwise, $(X_{1,i},\dots, X_{n,i})$ is uniformly distributed in $\{\pm 1\}^{n}$. 

With this observation in mind, we introduce a random variable $Z_X^\train = Z\sim \BSC_{\frac{1}{2} - \xi}(X)$ for some $\xi$ to be specified. Consider the following post-processing $\Phi^\train(Z)$ that produces a sequence of example $\tilde{X}_1,\dots, \tilde{X}_n$ as follows:
\begin{itemize}
    \item Independently for each coordinate $i\in [d]$, with probability $\nu + (1-\nu) 2^{-n+1}$, set all of $\tilde{X}_{1,i},\dots, \tilde{X}_{n,i}$ as $Z_i$. Otherwise choose $(\tilde{X}_{1,i},\dots, \tilde{X}_{n,i})\sim \Ucal(\{\pm 1\}^{n}\setminus \{(-1)^n,(+1)^n\})$ uniformly at random.
\end{itemize}

\paragraph*{The correctness of post-processing.} We compare the marginal distribution of $(X, \tilde{X}_{1:n})$ and $(X, X_{1:n})$. Let us prove that they are identically distributed for a proper choice of $\xi$. First of all, we should note that due to definition of $\theta, X, Z, X_{1:n}$ and $\Phi^\train$, we have that $(X,\tilde{X}_{1:n})$ and $(X, X_{1:n})$ are both product distribution over $(\{\pm 1\}^{1 + n})^d$ (i.e., the $d$ coordinates of the test example and the $n$ examples can be understood as independently generated). Next, observe that for each $j\in [d]$, the probability that all of $(X_{1,j},\dots,X_{n,j})$ are the same is $\nu + (1-\nu) 2^{-n+1}$, which is also the case for $\tilde{X}_{1:n}$. Furthermore, conditioning on the event that $(X_{j,i})_{j=1}^n$ are not all-$0$ nor all-$1$, they are uniformly distributed in $\{\pm 1\}^{n} \setminus \{(-1)^n,(+1)^n\}$, which is, again, also the case for $\tilde{X}_{1:n}$.

We turn to the case that all of $X_{*,j}$ are the same. It is easy to calculate that the probability of the event $X_{*,j} = (X_{0,i})^n$
is $\nu + (1-\nu) 2^{-n}$, while the same probability term w.r.t.~$\tilde{X}$ evaluates to $(\frac{1}{2}+\xi)\cdot (\nu + (1-\nu)2^{-n+1})$. It remains to set $\xi$ properly so that the two probability quantities coincide. Indeed, this gives
$$
\xi = \frac{\nu + (1-\nu) 2^{-n}}{\nu + (1-\nu)2^{-n+1}} - \frac{1}{2} = \frac{\nu}{2\nu + (1-\nu)2^{-n+2}}.
$$
Recall that we set $\nu =C/\sqrt{d}$. Therefore, for $n \le  \log_2 (\sqrt{d}/C)$, it holds that $\xi  \leq \nu 2^n$.

We have shown that with the proper choice of $\xi$, we can define $Z\sim \BSC_{\frac{1}{2}-\xi}(X)$ and post-process $Z$ to obtain a sequence of samples $\tilde{X}_{1:n}$ identically distributed as $X_{1:n}$. By Lemma~\ref{lem:simple-postprocess-SDPI}, this implies that the pair $(X,X_{1:n})$ enjoys at least the same SDPI as the pair $(X,Z)$. Quantitatively, the SDPI parameter we achieve is $\rho_n=\xi^2  \leq \nu^2 2^{2n} = \frac{C^2 2^{2n}}{d}$. By Theorem~\ref{thm: excess mem_direct}, for a constant $\alpha < 1/2$ this consequently implies that
$$
\mem_n(\bP,\alpha)
\geq \frac{C_{\alpha}}d{C^2 2^{2n}}  - I(\Acal(X_{1:n});\theta) ~.
 $$

In particular,  there exists a constant $c_\alpha$ such that for any $n\le \log_2 d/2 - c_\alpha$ we have that $\frac{C_{\alpha}}d{C^2 2^{2n}} \geq 2 C \sqrt{d}\log d$. 
Noting that $$I(\Acal(X_{1:n});\theta) \leq H(\theta) = d (-(1-\nu)\log(1-\nu) - \nu (\log(\nu/2))) = (1+o(1))C \sqrt{d} \log(\sqrt{d}/C) ,$$
we obtain
$$
\mem_n(\Acal, \bP_{\mathrm{sB}})\ge  \frac{C_{\alpha}}d{C^2 2^{2n}} - H(\theta) = \Omega\left( \frac{d}{2^{2n}}\right).
$$

\subsection{Sparse Boolean memorization upper bound}

Similarly to the other memorization upper bound proofs, given a sample $X_{1:n}\sim\Pcal_\theta^n$, we will argue that an algorithm can return a high accuracy classifier $h=\Acal(X_{1:n})$ which is describable using $\widetilde{O}(d/2^{2n})$ bits, hence in particular
\begin{equation*}
\mem_n(\Acal,\bP_{\mathrm{sB}})
=I(h;X_{1:n}\mid \theta) \leq  I(h;X_{1:n}) \leq H(h)=\widetilde{O}\left(\frac{d}{2^{2n}}\right)
~.
\end{equation*}
Let $\hat{S}\subset[d]$ be the subset that includes coordinates $j\in[d]$ in which the dataset $X_{1:n}$ is constant, namely for which $(X_i)_j=(X_{i'})_j$ is the same for all $i,i'\in[n]$. Furthermore, let $\hat{S}^\ell$ be a uniformly random selection of $\ell$ coordinates in $\hat{S}$, and define
\[
h(X)
=\one{\inner{X_{\hat{S}^\ell},(X_1)_{\hat{S}^\ell}}\geq t}
=\one{\sum_{j\in \hat{S}^\ell}X_j\cdot (X_1)_j\geq t}~.
\]
We will argue that for suitable
\[
\ell=\Theta\left(\frac{d}{2^{2n}}\right),
~\mbox{and}~~t=\sqrt{2\log(200)\ell}=\Theta(\sqrt{\ell}),
\]
this classifier has the desired properties. We note that the memorization upper bound readily follows by the description length argument above, since all that needs to be stored are
the $\ell$ bits of $(X_1)_{\hat{S}^\ell}$ alongside their corresponding indices in $[d]$, which requires $O(\ell\log(d))$ bits.

It remains to show that $h$ is a high accuracy classifier. We first note that in the null case $\Pcal_0=\Ucal(\{\pm1\}^d)$ and so it holds that
\[
\E_{X\sim \Pcal_0}\left[\inner{X_{\hat{S}^\ell},(X_1)_{\hat{S}^\ell}}\right]=0~,
\]
and by Hoeffding's bound
\begin{equation} \label{eq: sB null case}
\Pr_{X\sim \Pcal_0}\left[
\inner{X_{\hat{S}^\ell},(X_1)_{\hat{S}^\ell}} \geq t\right]
=\Pr_{X\sim \Pcal_0}\left[
\inner{X_{\hat{S}^\ell},(X_1)_{\hat{S}^\ell}} \geq \sqrt{2\log(200)\ell}
\right]
\leq 0.005~.
\end{equation}

We now turn to argue about the case $X\sim \Pcal_\theta$.
Note that for any coordinate $j\in\hat{S}^\ell$, if $j\in S$ then $X_j=(X_1)_j$ with probability $1$, yet if $j\notin S$ then $X_j\cdot (X_1)_j\sim\Ucal(\{\pm1\})$.
Hence, we see that
\begin{equation} \label{eq: E is cap of S}
    \E_{X\sim\Pcal_\theta}\left[\inner{X_{\hat{S}^\ell},(X_1)_{\hat{S}^\ell}}\right]
=|\hat{S}^\ell\cap S|
~.
\end{equation}
Therefore, we set out to estimate the size of $\hat{S}^\ell\cap S$. This set corresponds to coordinates in $\hat{S}^\ell$ which are ``truly'' constant, as opposed to uniformly random coordinates that happened to be constant on all seen examples, even thought they are not in $S$.
Accordingly, since $S\subset \hat{S}$, for any coordinate $j\in[d]:$
\[
\Pr[ j\in S \mid j\in \hat{S} ]
=\frac{\Pr[j\in \hat{S}\cap j\in S]}{\Pr[j\in \hat{S}]}
=\frac{\Pr[j\in S]}{\Pr[j\in \hat{S}]}
=  \frac{\nu}{\nu + (1-\nu)2^{-n+1}}
~.
\]
Moreover, by symmetry,
choosing $\ell$ coordinates in $\hat{S}$ into $\hat{S}^\ell$ is equivalent to sub-sampling $\ell$ coordinates among the constant coordinates, and therefore
\[
\E\left[ |\hat{S}^\ell\cap S|\right] = \ell\cdot \Pr[ j\in S \mid j\in \hat{S} ] 
= \ell \cdot \left( \frac{\nu}{\nu + (1-\nu)2^{-n+1}} \right) \geq_{(1)} \ell\cdot\frac{2^n}{4\sqrt{d}}
\ge_{(2)}
10\sqrt{2\log(200)}\sqrt{\ell}
~,
\]
where $(1)$ holds as long as $n\leq \log_2((1-\nu)\sqrt{d})=O(\log d)$ which we assume, and $(2)$ holds for $\ell=\frac{3200\log(200)d}{2^{2n}}$. Since the random variable $|\hat{S}^\ell\cap S|$ as a sum of independent indicators, Hoeffding's bound further ensures that
\[
\Pr\left[ |\hat{S}^\ell\cap S| >3t\right]>0.999~,
\]
and therefore under this probable event and applying yet another Hoeffding bound over $\inner{X_{\hat{S}^\ell},(X_1)_{\hat{S}^\ell}}$, \eqref{eq: E is cap of S} ensures that
\[
\Pr_{X\sim\Pcal_\theta}[\inner{X_{\hat{S}^\ell},(X_1)_{\hat{S}^\ell}} \geq t]>0.995~.
\]

Overall, combined with \eqref{eq: sB null case}, this ensures that the predictor $h$ has classification error of at most $0.01$, completing the proof.

\section{Proofs for Section \ref{sec: multiclass}} \label{sec: proof multi}

\subsection{Proof of Theorem~\ref{thm: multi cluster error local to global}}

The proof follows that of \cite[Theorem 2.3]{feldman2020does}.
We can write
\begin{align*}
\err(\Acal, \bPmult \mid S) 
&= \E_{\pi\sim \Pi_p^k,\theta_{1:k}\sim \Psi^k}\E_{S\sim \widetilde{\Pcal}_{\theta_{1:k},\pi},h\gets \Acal(S)}\left[\left. \frac{1}{2} \Pr_{X\sim \Pcal_{\theta_{1:k},\pi}}[h(X) = 0] + \frac{1}{2} \Pr_{X\sim \Pcal_{0}}[h(X) = 1] ~\right|~ S\right] \\
&= \frac{1}{2} \E_{\pi\sim \Pi_p^k,\theta_{1:k}\sim \Psi^k}\E_{S\sim \widetilde{\Pcal}_{\theta_{1:k},\pi},h\gets \Acal(S)}\left[\left. \sum_{j\sim \pi} \Pr_{X\sim \Pcal_{\theta_j}}[h(X) = 0] + \Pr_{X\sim \Pcal_{0}}[h(X) = 1] ~\right|~ S\right]
\end{align*}
Using $Z_{n\#\ell}$ to denote the collection of clusters $u$ that appear exactly $\ell$ times in $S$, we have
\begin{align*}
&~~~~ \frac{1}{2} \E_{\pi\sim \Pi,\theta_{1:k}\sim \Psi^k}\E_{S\sim \widetilde{\Pcal}_{\theta_{1:k},\pi},h\gets \Acal(S)}\left[\left. \Pr_{X\sim \Pcal_{\theta,\pi}}[h(X) = 0]  + \Pr_{X\sim \Pcal_0}[h(X) = 1] ~\right|~ S\right] \\
&= \sum_{0\le \ell \le n} \errn(\Acal,S,\ell) \cdot \frac{1}{|Z_{n\#\ell}|}\cdot \sum_{u\in Z_{n\#\ell}} \E[\pi_u \mid S] \\
&= \sum_{0\le \ell \le n} \tau_\ell \cdot \errn(\Acal, S, \ell).
\end{align*}
In words, the term $\tau_{\ell}$ reflects the change of (distribution of) $\pi_u$ after conditioning on $S$, while the term $\errn$ reflects the change of $\theta$. It turns out, by calculation, that after conditioning on $S$, the error of the algorithm can be succinctly described by the formula above.

To prove the theorem, it remains to argue that
\begin{align}
    \mathrm{opt}(\bPmult \mid S) \le \tau_0 \cdot \errn(\Acal, S,0) + O(1/k).\label{equ:opt ub conditioned on S}
\end{align}
We describe an algorithm $\Acal^*$ here: for every cluster $u$ that has been seen at least once in $S$, the algorithm trains a binary-classification model to distinguish $\Pcal_{\theta_u}$ from $\Pcal_{0}$ with error at most $\frac{1}{k^2}$ (this is possible thanks to Assumption~\ref{assump: multi learning}). Then, upon receiving a query $X$, the algorithm tests whether $X$ was likely from $\Pcal_{\theta_u}$ for every seen cluster $u$. $\Acal^*$ outputs $1$ if at least one of the tests outputs $1$. and $0$ otherwise. By simple union bound, for $X\sim \Pcal_0$, the algorithm makes an error on $X$ with probability at most $O(\frac{1}{k})$. For $X\sim \Pcal_{\theta_u}$ where $u$-th cluster is present in $S$, the probability that $\Acal^*$ misclassifies $X$ is at most $\frac{1}{k^2}$. For every $X\sim \Pcal_{\theta_u}$ where $u\in Z_{n\#0}$, $\Acal^*$ might fail miserably on $X$. However, by Lemma~\ref{lemma: ell coefficient}, the weight of all such clusters, after conditioning on $S$, is at most $\tau_0\cdot |Z_{n\# 0}|$. This shows that the error of $\Acal^*$ is upper bounded by the right-hand side of \eqref{equ:opt ub conditioned on S}, as claimed.

To establish the second item of the theorem, we simply take the average over $S$.

\subsection{Proof of Theorem~\ref{thm: accuracy to memorization}}

    Let $S = ((X_1,i_1),\dots, (X_n,i_n))$. We start by writing
    \begin{align*}
        I(\Acal(S) ; S \mid \theta_{1:k}, \pi)
        &= I(\Acal(S) ; S, (i_1,\dots, i_n) \mid \theta_{1:k}, \pi) \\
        &= I(\Acal(S); (i_1,\dots, i_n) \mid \theta_{1:k}, \pi) + I(\Acal(S); S \mid \theta_{1:k}, \pi, (i_1,\dots, i_n)) \\
        &\ge I(\Acal(S); S \mid \theta_{1:k}, \pi, (i_1,\dots, i_n)).
    \end{align*}
    We focus on the last line. By definition of conditional mutual information, we can first sample and condition on $\pi \sim \Pi_p^k$ and $i_1,\dots, i_n \sim \pi^n$, and consider the conditional mutual information $I(\Acal(S);S \mid \theta_{1:k})$ (where $i_1,\dots, i_n$, $\pi$ have been fixed). Note that the parameters $\theta_1,\dots, \theta_n$ and examples $X_1,\dots, X_n$ are independent of these choices and still have the same distribution. We denote by $S_\Xcal := (X_1,\ldots,X_n)$ the unconditioned part of $S$.

    Next, we want to show that the memorization of different clusters is essentially independent. Toward that goal, let $S_j$ be the set of pairs $(X,i)$ from $S$ such that $i = j$ (i.e.,~$S_j$ contains examples of $S$ that are from the $j$-th cluster). We have
    \begin{align*}
       I(\Acal(S) ; S \mid \theta_{1:k})
       &= H(S \mid \theta_{1:k}) - H(S \mid \theta_{1:k}, \Acal(S)) \\
       &\ge \sum_{j\in [k]} H(S_j \mid \theta_j) - H(S_j \mid \theta_{1:k}, \Acal(S)) &\text{(sub-additivity and additivity of entropy)} \\
       &\ge \sum_{j\in [k]} H(S_j \mid \theta_j) - H(S_j \mid \theta_{j},\Acal(S)) & \text{(monotonicity of conditional entropy)} \\
       &= \sum_{j\in [k]} I(\Acal(S) ; S_j \mid \theta_j).
    \end{align*}

    We now derive the claimed lower bound. As before, let $I_{n\# \ell}(S)$ be the collection of all $j$'s that appear exactly $\ell$ times among $i_1,\dots, i_n$.
    
    Having conditioned on $i_1,\dots,i_n$, let $j\in I_{n\#\ell}(S)$. We consider how well the algorithm distinguishes $\Pcal_{\theta_j}$ from $\Pcal_0$ (when we take the average over $\theta_j, S_j$ and $\Acal(S)$). Thinking of this as a binary classification problem between $\Pcal_{\theta_j}$ and $\Pcal_0$ where $\ell$ samples from $\Pcal_{\theta_j}$ are available, we can evaluate the error of the algorithm by
    \begin{align*}
        e_j := \frac{1}{2} \E_{\theta_{1:k}\sim \Psi^k, S_\Xcal, h \gets \Acal(S)}\left[\Pr_{X\sim \Pcal_{\theta_j}}[h(X) = 0] + \Pr_{X\sim \Pcal_{0}}[h(X) = 1] \right].
    \end{align*}
    Using Assumption~\ref{assump: error memorization scale}, we obtain that
    \begin{align*}
        I(\Acal(S); S_j \mid \theta_j) \ge c_\bP (1 - 2e_j)\cdot \mem_\ell(\bP) .
    \end{align*}
    Note that $$\sum_{j\in I_{n\# \ell}(S)}e_j = \E_{\theta_{1:k}\sim \Psi^k, S_\Xcal}\left[\errn(\Acal, \theta_{1:k},S, \ell)\right] .$$ Adding up the contribution from different $j\in I_{n\# \ell}(S)$, we get
    \begin{align*}
        \sum_{j\in I_{n\# \ell}(S)} I(\Acal(S); S_j \mid \theta_j)\ge c_\bP \cdot \E_{\theta_{1:k}\sim \Psi^k, S_\Xcal}\left[ (|I_{n\#\ell}(S)| - 2 \cdot\errn(\Acal, \theta_{1:k}, S, \ell)) \cdot \mem_\ell(\bP) \right]~.
    \end{align*}    
Finally, summing over all $\ell$ and taking the average over $\pi, i_1,\dots, i_n$ concludes the proof.

\end{document}